\documentclass[twoside,11pt]{article}
\usepackage{jmlr2e}
\usepackage{color,amsmath, amssymb, epsfig, algorithm, algorithmic, caption, multirow, dsfont}

\usepackage[tight,footnotesize]{subfigure}
\usepackage{amssymb, epsfig,bm}

\newcommand{\G}{\mathcal{G}}

\newcommand{\N}{\mathcal{N}}
\newcommand{\R}{\mathbb{R}}
\newcommand{\E}{\mathbb{E}}

\newcommand{\vct}[1]{\bm{#1}}
\newcommand{\mtx}[1]{\bm{#1}}
\newcommand{\vx}{\vct{x}}

\newcommand{\vy}{\vct{y}}
\newcommand{\vz}{\vct{z}}
\newcommand{\vv}{\vct{v}}
\newcommand{\va}{\vct{a}}

\newcommand{\vw}{\vct{w}}

\newcommand{\vg}{\vct{g}}
\newcommand{\vphi}{\vct{\phi}}

\newcommand{\mPhi}{\mtx{\Phi}}

\newcommand{\mX}{\mtx{X}}

\newcommand{\T}{\mathcal{T}}
\newcommand{\W}{\mathcal{W}}

\renewcommand{\Pr}{\mathbb{P}}

\newcommand{\beq}{\begin{equation}}
\newcommand{\eeq}{\end{equation}}
\newcommand{\fll}{\frac{\lambda_1}{\sqrt{l}}}

\newcommand{\fla}{\frac{\lambda_1}{\sqrt{\alpha}}}

\begin{document}

\title{Classification with Sparse Overlapping Groups}
\author{\name Nikhil S. Rao \email nrao2@wisc.edu \\ 
\name Robert D. Nowak \email nowak@ece.wisc.edu \\
       \addr Department of Electrical and Computer Engineering\\
       University of Wisconsin-Madison\\
       \AND
       \name Christopher R. Cox \email crcox@wisc.edu \\
       \name Timothy T. Rogers \email ttrogers@wisc.edu \\
       \addr Department of Psychology\\
       University of Wisconsin-Madison}

\editor{}

\maketitle

\begin{abstract}
Classification with a sparsity constraint on the solution plays a central role in many high dimensional machine learning applications. In some cases, the features can be grouped together, so that entire subsets of features can be selected or not selected. In many applications, however, this can be too restrictive. In this paper, we are interested in a less restrictive form of structured sparse feature selection: we assume that while features can be grouped according to some notion of similarity, not all features in a group need be selected for the task at hand. When the groups are comprised of disjoint sets of features, this is sometimes referred to as the ``sparse group'' lasso, and it allows for working with a richer class of models than traditional group lasso methods.  Our framework generalizes conventional sparse group lasso further by allowing for overlapping groups, an additional flexiblity needed in many applications and one that presents further challenges. The main contribution of this paper is a new procedure called {\em Sparse Overlapping Group (SOG) lasso}, a convex optimization program that automatically selects similar features for classification in high dimensions. We establish model selection error bounds for SOGlasso  classification problems under a fairly general setting.  In particular, the error bounds are the first such results for classification using the sparse group lasso.  Furthermore, the general SOGlasso bound specializes to results for the lasso and the group lasso, some known and some new. The SOGlasso is motivated by multi-subject fMRI studies in which functional activity is classified using brain voxels as features, source localization problems in Magnetoencephalography (MEG), and analyzing  gene activation patterns in microarray data analysis.  Experiments with real and synthetic data demonstrate the advantages of SOGlasso compared to the lasso and group lasso.  
\end{abstract}

%%%%%%%%%%%%%%%%%%%%%%%%%%%%%%%%%%%%%%%%%%%%%%%%%%
%%%%%%%%%%%%%%%%%%%%%%%%%%%%%%%%%%%%%%%%%%%%%%%%%%
%%%%%%%%%%%%%%%%%%%%%%%%%%%%%%%%%%%%%%%%%%%%%%%%%%
%%%%%%%%%%%%%%%%%%%%%%%%%%%%%%%%%%%%%%%%%%%%%%%%%%
%%%%%%%%%%%%%%%%%%%%%%%%%%%%%%%%%%%%%%%%%%%%%%%%%%
%%%%%%%%%%%%%%%%%%%%%%%%%%%%%%%%%%%%%%%%%%%%%%%%%%
%%%%%%%%%%%%%%%%%%%%%%%%%%%%%%%%%%%%%%%%%%%%%%%%%%
%%%%%%%%%%%%%%%%%%%%%%%%%%%%%%%%%%%%%%%%%%%%%%%%%%
%%%%%%%%%%%%%%%%%%%%%%%%%%%%%%%%%%%%%%%%%%%%%%%%%%
%%%%%%%%%%%%%%%%%%%%%%%%%%%%%%%%%%%%%%%%%%%%%%%%%

\section{Introduction}
\label{sec:intro}

Binary classification plays a major role in many machine learning and signal processing applications. In many modern applications where the number of features far exceeds the number of samples, one typically wishes to select only a few features, meaning only a few coefficients are non zero in the solution \footnote{a zero coefficient in the solution implies the corresponding feature is not selected}. This corresponds to the case of searching for sparse solutions. The notion of sparsity prevents over-fitting and leads to more interpretable solutions in high dimensional machine learning, and has been extensively studied in \citep{bachlogistic, plan1bit, Mest, buneahonest}, among others. 

In many applications, we wish to impose structure on the sparsity pattern of the coefficients recovered.  In particular, often it is known a priori that the optimal sparsity pattern will tend to involve clusters or groups of coefficients, corresponding to pre-existing groups of features.  The form of the groups is known, but the subset of groups that is relevant to the classification task at hand is unknown.  This prior knowledge reduces the space of possible sparse coefficients thereby potentially leading to better results than simple lasso methods. In such cases, the group lasso, with or without overlapping groups \citep{yuanlin} is used to recover the coefficients. The group lasso forces all the coefficients in a  group to be active at once: if a coefficient is selected for the task at hand, then all the coefficients in that group are selected. When the groups overlap, a modification of the penalty allows one to recover coefficients that can be expressed as a union of groups \citep{jacob, latent}.

While the group lasso has enjoyed tremendous success in high dimensional feature selection applications,  we are interested in a much less restrictive form of structured feature selection for classification. Suppose that the features can be arranged into (possibly) \emph{overlapping} groups based on some notion of similarity, depending on the application. 
The notion of similarity can be loosely defined, and it is used to reflect the prior knowledge that if a feature is relevant for the learning task at hand, then features similar to it may also be relevant.  It is known that while many features may be similar to each other, not all similar features are relevant for the specific learning problem. We propose a new procedure called Sparse Overlapping Group (SOG) lasso to reflect this form of structured sparsity.

As an example, consider the task of identifying relevant genes that play a role in predicting a disease. Genes are organized into pathways \citep{pathway}, but not every gene in a pathway might be relevant for prediction. At the same time, it is reasonable to assume that if a gene from a particular pathway is relevant, then other genes from the same pathway may also be relevant. In such applications, the group lasso may be too constraining while the lasso may be too under-constrained.

%%%%%%%%%%%%%%%%%%%%%%%%%%%%%%%%%%%%%%%%%%%%%%%%%%
%%%%%%%%%%%%%%%%%%%%%%%%%%%%%%%%%%%%%%%%%%%%%%%%%%
%%%%%%%%%%%%%%%%%%%%%%%%%%%%%%%%%%%%%%%%%%%%%%%%%%
%%%%%%%%%%%%%%%%%%%%%%%%%%%%%%%%%%%%%%%%%%%%%%%%%%
%%%%%%%%%%%%%%%%%%%%%%%%%%%%%%%%%%%%%%%%%%%%%%%%%%
%%%%%%%%%%%%%%%%%%%%%%%%%%%%%%%%%%%%%%%%%%%%%%%%%%
%%%%%%%%%%%%%%%%%%%%%%%%%%%%%%%%%%%%%%%%%%%%%%%%%%
%%%%%%%%%%%%%%%%%%%%%%%%%%%%%%%%%%%%%%%%%%%%%%%%%%
%%%%%%%%%%%%%%%%%%%%%%%%%%%%%%%%%%%%%%%%%%%%%%%%%%
%%%%%%%%%%%%%%%%%%%%%%%%%%%%%%%%%%%%%%%%%%%%%%%%%%

\subsection{Model and  Results}

We first present the main results of this paper at a glance. Uppercase and lowercase bold letters indicate matrices and vectors respectively. We assume a sparse learning framework, with a feature matrix $\mPhi \in \R^{n \times p}, ~\ n \ll p$. We assume each element of $\mPhi$ to be distributed as a standard Gaussian random variable. Assuming the data to arise from a Gaussian distribution simplifies analysis, and allows us to leverage tools from existing literature. Later in the paper, we will allow for correlations in the features as well, reflecting a more realistic setting.  In the results that follow, $C$ is a positive constant, the value of which can be different from one result to the other. 

We focus on binary  classification settings, and assume that each observation $\vy_i \in \{-1, +1 \}, ~\ i = 1,2,\ldots, n$ are randomly distributed according to the model \citep{plan1bit}
\begin{equation}
\label{ymodel}
\mathbb{E}[\vy_i | \vphi_i] = f(\langle \vphi_{i} , \vx^\star \rangle), 
\end{equation}
where $\vphi_i$ is the $i^{th}$ row of $\mPhi$ corresponding to the features of data $i$,  $\vx^\star$ is the true coefficient vector of interest, and $f$ is a function mapping from $\R$ to $[-1,+1]$.  The argument of $f$ is the Euclidean inner product: $\langle \vphi_{i} , \vx^\star \rangle  = \vphi_i^T \vx^\star$. The function $f$ need not be known precisely. We only assume that it satisfies for $g \sim \N(0,1)$
\begin{equation}
\label{defe}
\mathbb{E}[f(g)g] > 0 \ .
\end{equation}

Without loss of generality, we assume $\vx^\star$ to have unit Euclidean norm, since the normalization can be absorbed into the function $f$.
The value 
\begin{equation}
\label{deflam}
\sigma_f := 1/\mathbb{E}[f(g)g]
\end{equation}
 quantifies the strength of the correlation between the labels $y_i$ and inner products $\langle \vphi_i,\vx^\star\rangle$. It plays the role of the noise level and will appear in
our error bounds, but it need not be known to compute our proposed estimator of $\vx^\star$.

This set-up allows for the consideration of a very general setting for classification, and subsumes many interesting cases. For example, the logistic model
\begin{equation}
\label{ymodel_logit}
\Pr \left( \vy_{i} = 1 \right) = \frac{\exp(\beta\, \langle \vphi_{i} , \vx^\star \rangle)}{1 + \exp(\beta\, \langle \vphi_{i} , \vx^\star \rangle)} \ , 
\end{equation}
is equivalent to 
\[
f(\langle \vphi_i, \vx^\star \rangle) = \mbox{tanh}(\beta \langle \vphi_i, \vx^\star  \rangle),
\]
for any constant $\beta>0$, yielding a corresponding \footnote{ See Corollary 3.3 in \citep{plan1bit} for a derivation}
\[\sigma_f  = \frac{2}{\beta} \mathbb{E} [\mbox{sech}^2 (\beta \vg/2)]^{-1} \leq \frac{6}{\min\{\beta,1\}}.\]
The constant $\beta$ accounts  for the fact that we consider $\| \vx^\star \| = 1$. Indeed, for a general vector $\vz$, we can write $\langle \vphi_i, \vz \rangle  = \beta \langle \vphi_i, \frac{\vz}{\| \vz \|}   \rangle $ , where $\beta = \| \vz \|$. The second argument in the inner product is now a vector of unit norm, and it gives rise to the expression in (\ref{ymodel_logit}).

The framework also allows for the quantized 1-bit measurement model
\[
f(\langle \vphi_{i} , \vx^\star \rangle) = \mbox{sign}(\langle \vphi_{i} , \vx^\star \rangle) \ ,
\]
which can be seen as the limiting case of the logistic model as $\beta \rightarrow \infty$, and with $\sigma_f = \sqrt{\frac{\pi}{2}}$ 

We work with this general formulation since for classification problems of interest, the logistic (or any other) model may be chosen somewhat arbitrarily.  Existing theoretical results often apply only to the chosen model (for example, \citep{Mest}).  Our estimator only requires the observations are correlated with the features, in the sense of (\ref{deflam}). In any such case, the underlying form of $f$ need not be known to compute our estimator of $\vx^\star$ and will enter in the error bounds only through $\sigma_f$.

We are interested in the following form of structured sparsity.  Assume that the features can be organized into $K$ \emph{possibly overlapping} groups, each consisting of $L$ features, based on a user-defined measure of similarity, depending on the application.  Moreover, assume that if a certain feature is relevant for the learning task at hand, then features similar to it may also be relevant.  Note that we assume groups of equal size $L$ for convenience.  It is easy to relax this assumption.   These assumptions suggest a structured pattern of sparsity in the coefficients wherein a subset of the groups are relevant to the learning task, and within the relevant groups a subset of the features are selected.  In other words, $\vx^\star \in \R^{ p}$ has the following structure:
\begin{itemize}
\item its support is localized to a union of a subset of the groups, and
\item its support is localized to a sparse subset within each such group
\end{itemize}

Armed with these preliminaries, we state the theoretical sample complexity bounds proved later in the paper. 

\vspace{2mm}

\emph{If $\vx^\star \in \R^p$ has $k \leq K$ non-zero groups and $l \leq L$ coefficients non-zero within each non-zero group, then $\mathcal{O}(\sigma_f^2 \, k \left[ \log(\frac{K}{k} ) + l \log(\frac{L}{l}) + l \right] )$ independent Gaussian measurements of the form (\ref{ymodel}) are sufficient to accurately estimate $\vx^\star$ by solving a convex program.}
\vspace{2mm}

The statement above merits further explanation. We show in Lemma \ref{mwnonconvset}  (Section 4.1) via a combinatorial argument that to estimate any vector with parameters $k, K , l,  L$ as stated above, $\mathcal{O}(\sigma_f^2k \left[ \log(\frac{K}{k} ) + l \log(\frac{L}{l}) + l \right] ) $ samples would suffice. However, looking for vectors with these properties amounts to solving a non-convex program. When the groups do not overlap, we show that the solution to a \emph{convex} program also succeeds in accurately estimating $\vx^\star$ using the same number of measurements.  When the groups overlap, we show that the measurement bound holds with an additional factor of $R^2$, where $R\geq 1$ is the maximum number of groups that contain any one feature. In most applications of interest (e.g., the fMRI example discussed below), $R$ is small. Nonetheless, we also show that no matter how many groups a particular coefficient belongs to, the sample complexity of our proposed estimator is never greater than that of the standard lasso or the overlapping group lasso. These statements will be made more precise in the sequel.

%%%%%%%%%%%%%%%%%%%%%%%%%%%%%%%%%%%%%%%%%%%%%
%%%%%%%%%%%%%%%%%%%%%%%%%%%%%%%%%%%%%%%%%%%%%
%%%%%%%%%%%%%%%%%%%%%%%%%%%%%%%%%%%%%%%%%%%%%
%%%%%%%%%%%%%%%%%%%%%%%%%%%%%%%%%%%%%%%%%%%%%
%%%%%%%%%%%%%%%%%%%%%%%%%%%%%%%%%%%%%%%%%%%%%
%%%%%%%%%%%%%%%%%%%%%%%%%%%%%%%%%%%%%%%%%%%%%
%%%%%%%%%%%%%%%%%%%%%%%%%%%%%%%%%%%%%%%%%%%%%
%%%%%%%%%%%%%%%%%%%%%%%%%%%%%%%%%%%%%%%%%%%%%
%%%%%%%%%%%%%%%%%%%%%%%%%%%%%%%%%%%%%%%%%%%%%
%%%%%%%%%%%%%%%%%%%%%%%%%%%%%%%%%%%%%%%%%%%%%

\subsection{Motivation: The SOGlasso for Multitask Learning}
\label{sec:mtl}
The SOG lasso is motivated in part by multitask learning applications. The group lasso is a commonly used tool in multitask learning, and it encourages the same set of features to be selected across all tasks. We wish to focus on a less restrictive version of multitask learning, where the main idea is to encourage selection of features that are similar, but not identical, across tasks. This is accomplished by defining subsets of similar features and searching for solutions that select only a few subsets (common across tasks) and a sparse number of features within each subset (possibly different across tasks). Figure \ref{fig:spatterns} shows an example of the patterns that typically arise in sparse multitask learning applications, along with the one we are interested in. 

\begin{figure}[!h]
\centering
\subfigure[Sparse]{
\includegraphics[width = 20mm, height = 29mm]{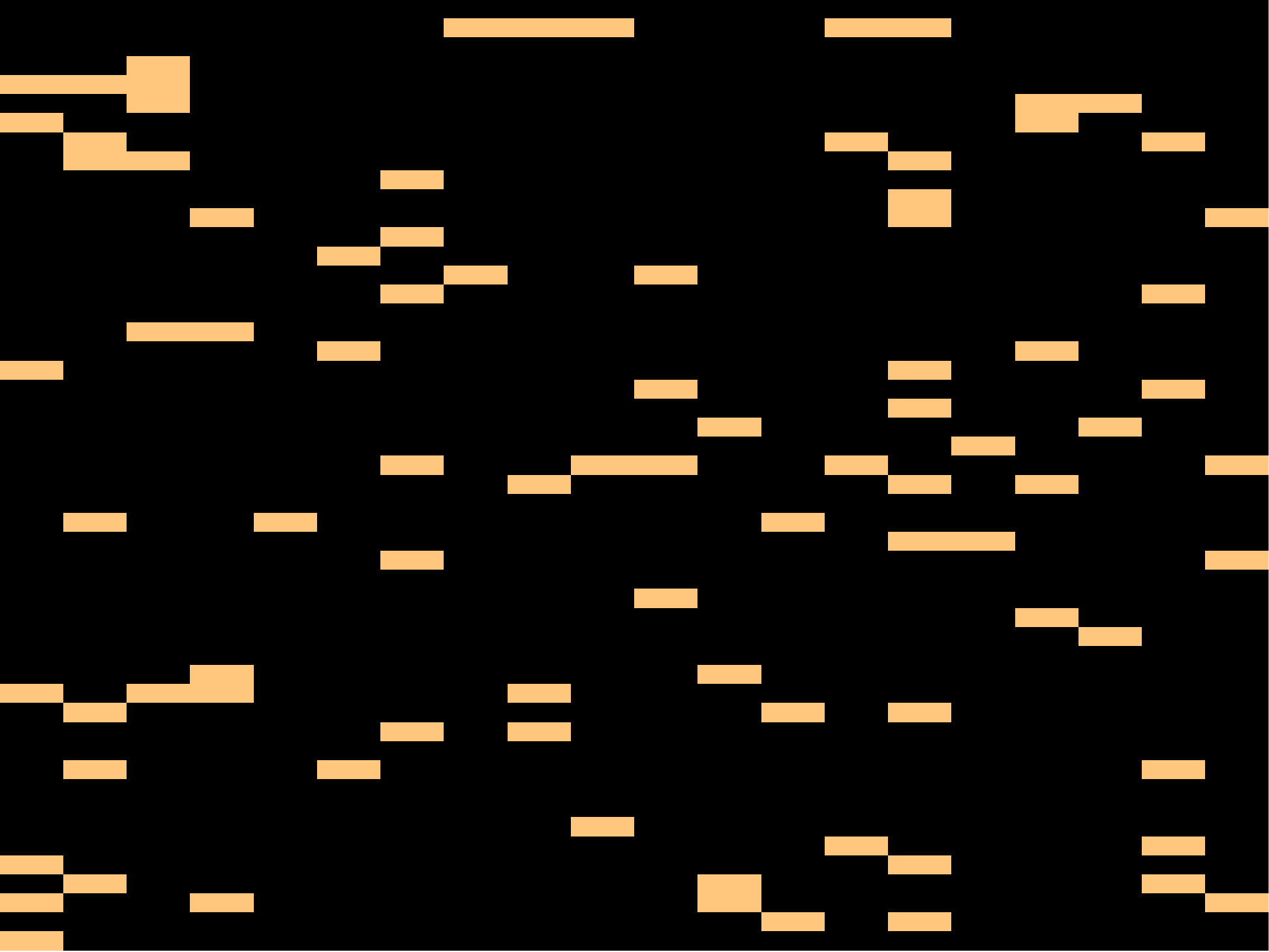}
\label{fig:lpattern}}
\qquad
\subfigure[Group sparse ]{
\includegraphics[width = 20mm, height = 29mm]{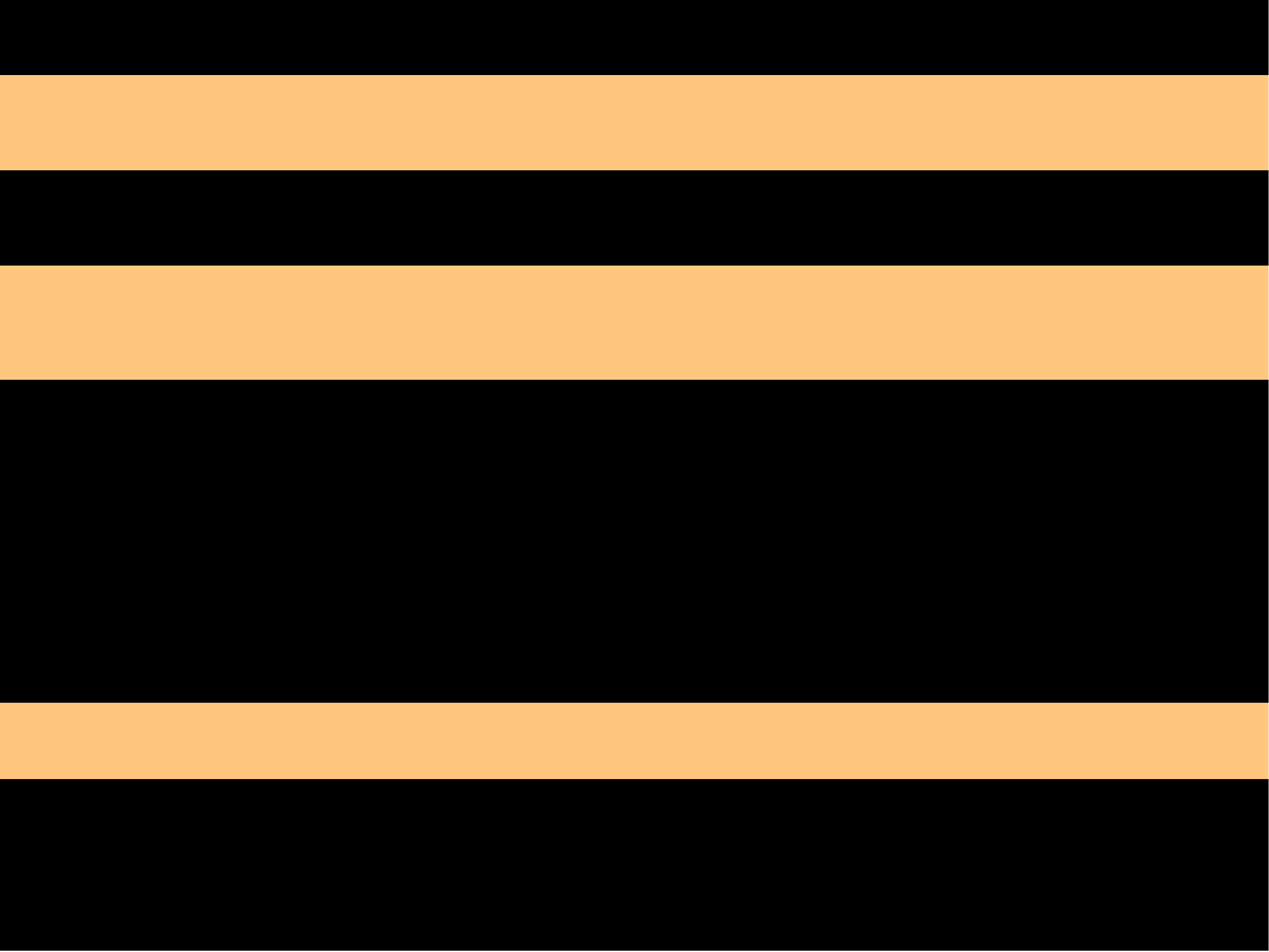}
\label{fig:glpattern}}
\qquad
\subfigure[Group sparse {\em plus} sparse ]{
\includegraphics[width = 20mm, height = 29mm]{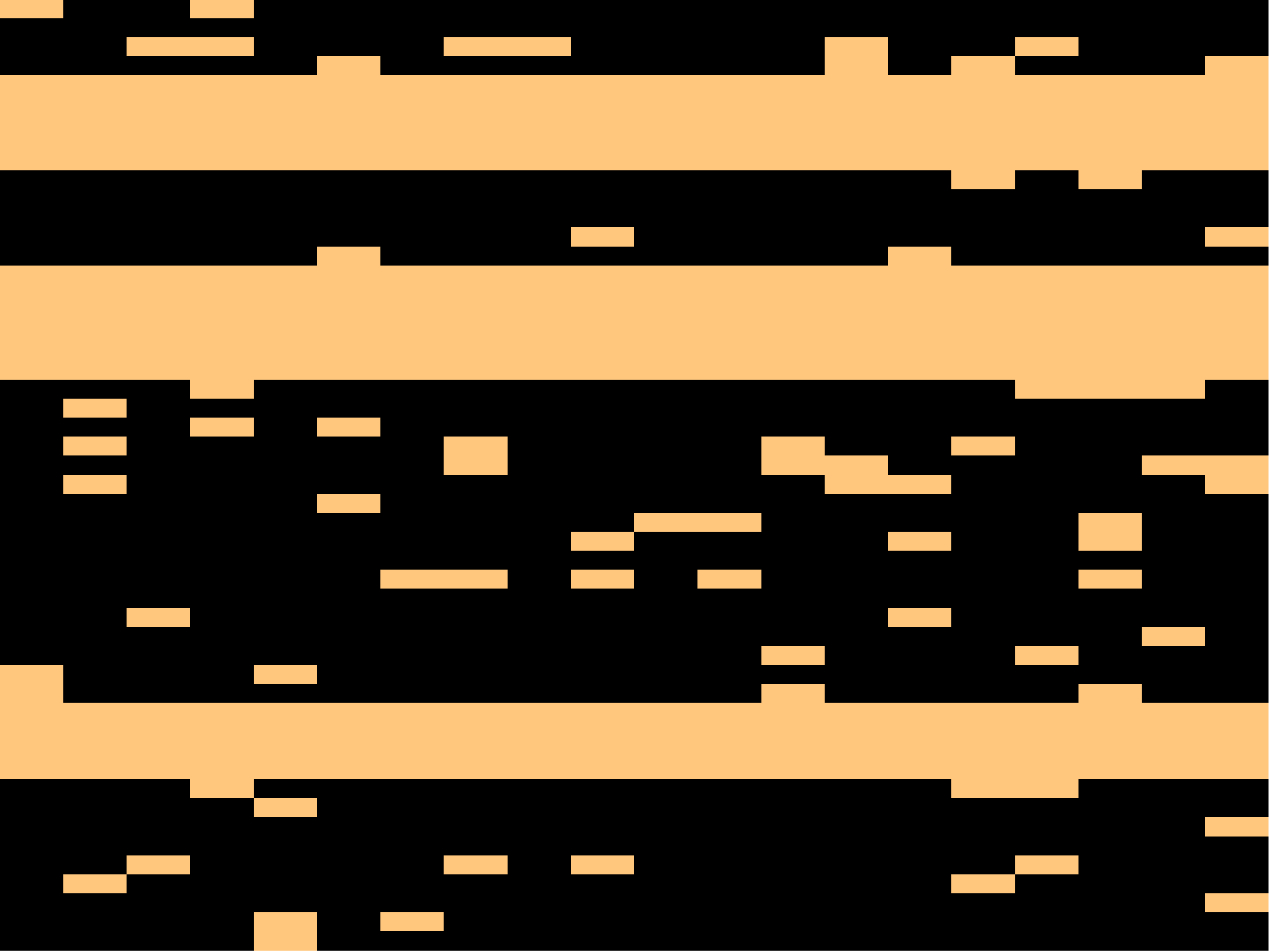}
\label{fig:dirtypattern}}
\qquad
\subfigure[Group sparse {\em and} sparse ]{
\includegraphics[width = 20mm, height = 29mm]{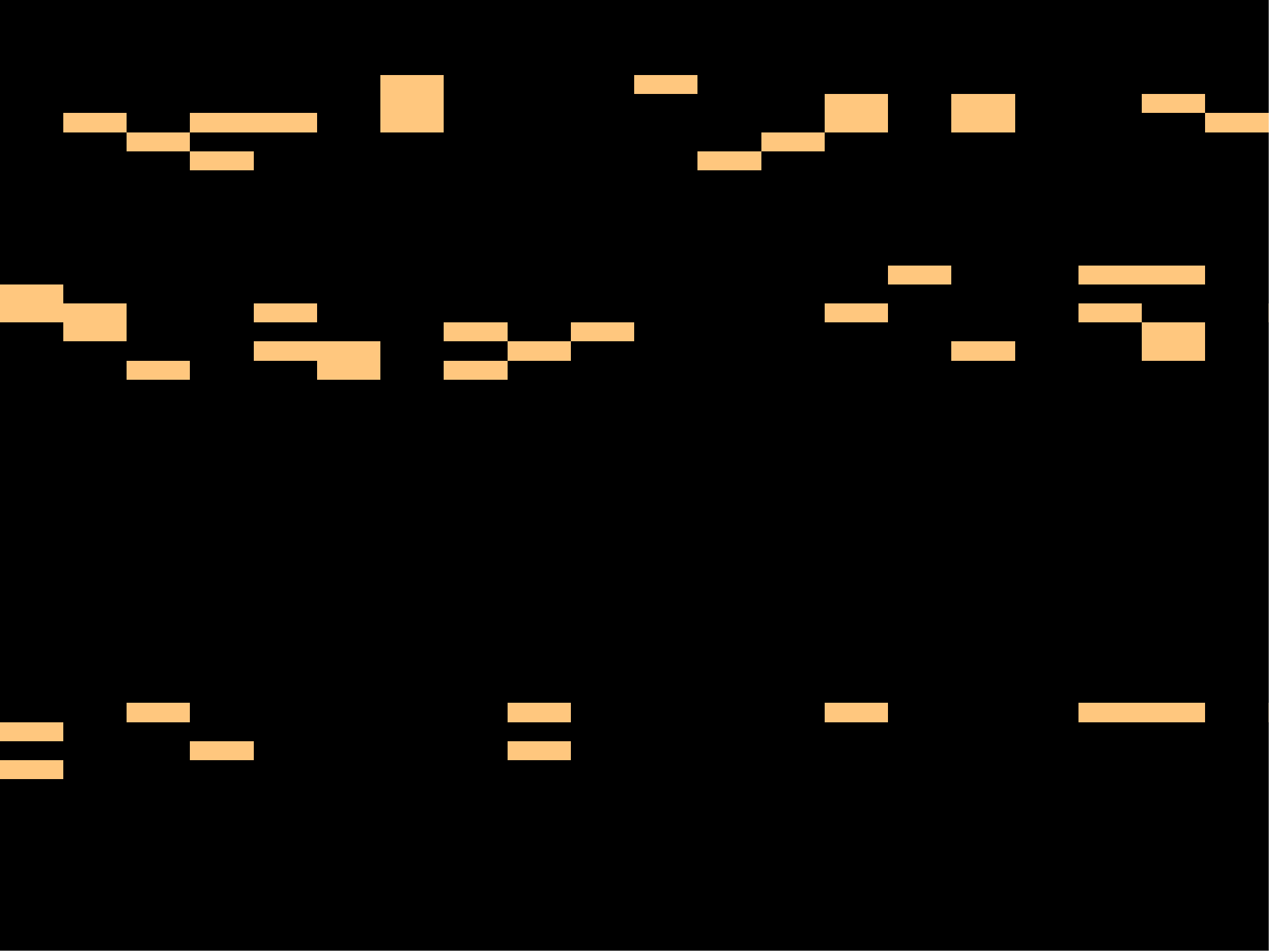}
\label{fig:sglpattern}}
\caption{ A comparison of different sparsity patterns in the multitask learning setting. Figure \subref{fig:lpattern} shows a standard sparsity pattern. An example of group sparse patterns promoted by Glasso \citep{yuanlin} is shown in Figure \subref{fig:glpattern}. In Figure \subref{fig:dirtypattern}, we show the patterns considered in \citep{dirty}. Finally, in Figure \subref{fig:sglpattern}, we show the patterns we are interested in this paper. The groups are sets of rows of the matrix, and can overlap with each other}
\label{fig:spatterns}
\end{figure}

A major application that we are motivated by is the analysis of multi-subject fMRI data, where the goal is to predict a cognitive state from measured neural activity using voxels as features. Because brains vary in size and shape, neural structures can be aligned only crudely. Moreover, neural codes can vary somewhat across individuals \citep{feredoes}. Thus, neuroanatomy provides only an approximate guide as to where relevant information is located across individuals: a voxel useful for prediction in one participant suggests the general anatomical neighborhood where useful voxels may be found, but not the precise voxel. Past work in inferring sparsity patterns across subjects has involved the use of groupwise regularization \citep{heskesgroupwise}, using the logistic lasso to infer sparsity patterns without taking into account the relationships across different subjects \citep{logitbrain}, or using the elastic net penalty to account for groupings among coefficients \citep{rishsparse}. These methods do not exclusively take into account both the common macrostructure and the differences in microstructure across brains, and the SOGlasso allows one to model both the commonalities and the differences across brains. Figure \ref{soslassobrain} sheds light on the motivation, and the grouping of voxels across brains into overlapping groups

\begin{figure}
\centering
\includegraphics[width = 120mm, height = 40mm]{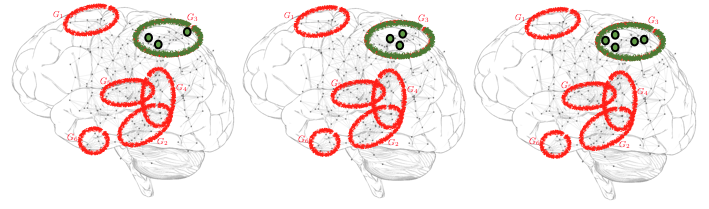}
\caption{SOGlasso for fMRI inference. The figure shows three brains, and voxels in a particular anatomical region are grouped together, across all individuals (red and green ellipses). For example, the green ellipse in the brains represents a single group. The groups denote anatomically similar regions in the brain that may be co-activated. However, within activated regions, the exact location and number of voxels may differ, as seen from the green spots. }
\label{soslassobrain}
\end{figure}

%%%%%%%%%%%%%%%%%%%%%%%%%%%%%%%%%%%%%%%%%%%%%
%%%%%%%%%%%%%%%%%%%%%%%%%%%%%%%%%%%%%%%%%%%%%
%%%%%%%%%%%%%%%%%%%%%%%%%%%%%%%%%%%%%%%%%%%%%
%%%%%%%%%%%%%%%%%%%%%%%%%%%%%%%%%%%%%%%%%%%%%
%%%%%%%%%%%%%%%%%%%%%%%%%%%%%%%%%%%%%%%%%%%%%
%%%%%%%%%%%%%%%%%%%%%%%%%%%%%%%%%%%%%%%%%%%%%
%%%%%%%%%%%%%%%%%%%%%%%%%%%%%%%%%%%%%%%%%%%%%
%%%%%%%%%%%%%%%%%%%%%%%%%%%%%%%%%%%%%%%%%%%%%
%%%%%%%%%%%%%%%%%%%%%%%%%%%%%%%%%%%%%%%%%%%%%
%%%%%%%%%%%%%%%%%%%%%%%%%%%%%%%%%%%%%%%%%%%%%

\subsection{Our Contributions}
In this paper, we consider binary classification with a constraint on the structure of the sparsity pattern of the coefficients. We assume that the coefficients can be arranged (according to a predefined notion of similarity) into (overlapping) groups.  Not only are only a few groups selected, but the selected groups themselves are also sparse. In this sense, our constraint can be seen as an extension of sparse group lasso \citep{sgl} for overlapping groups where the sparsity pattern lies in a union of groups. We are mainly interested in classification problems, but the method can also be applied to regression settings, by making an appropriate change in the loss function of course. We consider a union-of-groups formulation as in \citep{jacob}, but with an additional sparsity constraint on the selected groups. To this end, we analyze the Sparse Overlapping Sets (SOG) lasso, where the overlapping groups might correspond to coefficients of features arbitrarily grouped according to the notion of similarity.  We also consider a very general classification setting, and do not make restrictive assumptions on the observation model. 

We introduce a constraint that promotes sparsity patterns that can be expressed as a union of sparsely activated groups. We show that the constraint is a tight convex relaxation of the set of coefficients having the sparsity pattern we are interested in. The main contribution of this paper is a theoretical analysis of the model selection consistency of the SOGlasso estimator, under a very general binary classification setting. Based on certain parameter values, our method reduces to other known cases of penalization for sparse high dimensional recovery. Specifically, under the logistic regression model,  our method generalizes the group lasso \citep{meierglassologistic, jacob}, and also extends to handle groups that can arbitrarily overlap with each other. We also recover results for the lasso for logistic regression \citep{buneahonest, Mest, plan1bit, bachlogistic}.  In this sense, our work unifies the lasso, the group lasso as well as the sparse group lasso for to handle overlapping groups.  At the same time, our methods apply to settings beyond the logistic regression, to include a far richer class of models. To the best of our knowledge, this is the first paper that provides such a unified theory and sample complexity bounds for all these methods. 

In the case of linear regression and multitask learning, the authors in \citep{sprechmann, chilasso},  consider a similar situation with non overlapping subsets of features. 
We assume that the features can arbitrarily overlap. When the groups overlap, the methods mentioned above suffer from a drawback: entire groups are set to zero, in effect zeroing out many coefficients that might be relevant to the tasks at hand. This has undesirable effects in many applications of interest, and the authors in \citep{jacob} propose a version of the group lasso to circumvent this issue.

We also test our regularizer on both toy and real datasets. Our experiments reinforce our theoretical results, and demonstrate the advantages of the SOGlasso over standard lasso and group lasso methods, when the features can indeed be grouped according to some notion of similarity. We show that the SOGlasso is especially useful in multitask Functional Magnetic Resonance Imaging (fMRI) applications, and gene selection applications in computational biology. 

To summarize, the main contributions of this paper are the following:

\begin{enumerate}
\item {\bfseries New regularizers for structured sparsity:} We propose the Sparse Overlapping Group (SOG) lasso, a convex optimization problem that encourages the selection of coefficients that are both within-and across- group sparse. The groups can arbitrarily overlap, and the pattern obtained can be represented as a union of a small number of groups. This generalizes other known methods, and provides a common regularizer that can be used for any structured sparse problem with two levels of hierarchy \footnote{Further levels can also be added as in \citep{bachhierarchical}, but that is beyond the scope of this paper.}: groups at a higher level, and singletons at the lower level.
\item{\bfseries New theory for classification with structured sparsity:}  We provide a theoretical analysis for the model selection consistency of the SOGlasso estimator, for binary classification. The general results we obtain specialize to the lasso, the group lasso (with or without overlapping groups) and the sparse group lasso. We also make minimal assumptions on the measurement model, allowing for theory that is applicable in a wide range of classification settings. We obtain a bound on the sample complexity of the SOGlasso under both independent and correlated Gaussian measurement designs, and this in turn also translates to corresponding results for the lasso and the group lasso. In this sense, we obtain a unified theory for performing structured feature selection in high dimensions. 
\item{\bfseries Applications:}  A major motivating application for this work is the analysis of multi-subject fMRI data. We apply the SOGlasso to fMRI data, and show that the results we obtain not only yield lower errors on hold-out test sets compared to previous methods, but also lead to more interpretable results. To show it's applicability to other domains, we also apply the method to breast cancer data to detect genes that are relevant in the prediction of metastasis in breast cancer tumors. 
\end{enumerate}

In \citep{soslassonips}, the authors introduced the SOGlasso problem emphasizing the motivating fMRI application. The authors also derived theoretical consistency results under the \emph{linear regression} setting. This paper gives further embellishes the reasons for considering the SOGlasso penalty, and derives consistency results for \emph{classification}. We also present  some novel applications in computational biology where similar notions can be applied to achieve significant gains over existing methods. Our work here presents novel results for the group lasso with potentially overlapping groups as well as the sparse group lasso for classification settings, as special cases of the theory we develop.

%%%%%%%%%%%%%%%%%%%%%%%%%%%%%%%%%%%%%%%%%%%%%
%%%%%%%%%%%%%%%%%%%%%%%%%%%%%%%%%%%%%%%%%%%%%
%%%%%%%%%%%%%%%%%%%%%%%%%%%%%%%%%%%%%%%%%%%%%
%%%%%%%%%%%%%%%%%%%%%%%%%%%%%%%%%%%%%%%%%%%%%
%%%%%%%%%%%%%%%%%%%%%%%%%%%%%%%%%%%%%%%%%%%%%
%%%%%%%%%%%%%%%%%%%%%%%%%%%%%%%%%%%%%%%%%%%%%
%%%%%%%%%%%%%%%%%%%%%%%%%%%%%%%%%%%%%%%%%%%%%
%%%%%%%%%%%%%%%%%%%%%%%%%%%%%%%%%%%%%%%%%%%%%
%%%%%%%%%%%%%%%%%%%%%%%%%%%%%%%%%%%%%%%%%%%%%
%%%%%%%%%%%%%%%%%%%%%%%%%%%%%%%%%%%%%%%%%%%%%

\subsection{Past Work}
When the groups of features do not overlap, \citep{sgl} proposed the Sparse Group Lasso (SGL) to recover coefficients that are both within- and across- group sparse. SGL and its variants for multitask learning has found applications in character recognition \citep{chilasso, sprechmann}, climate and oceanology applications \citep{sglclimate}, and in gene selection in computational biology \citep{sgl}. In \citep{bachhierarchical}, the authors extended the method to handle tree structured sparsity patterns, and showed that the resulting optimization problem admits an efficient implementation in terms of proximal point operators. Along related lines, the exclusive lasso \citep{exlasso} can be used when it is explicitly known that features in certain groups are negatively correlated. When the groups overlap, \citep{jacob, latent} proposed a modification of the group lasso penalty so that the resulting coefficients can be expressed as a union of groups. They proposed a replication-based strategy for solving the problem, which has since found application in computational biology \citep{jacob} and image processing \citep{nricip11}, among others. The authors in \citep{glopridu} proposed a method to solve the same problem in a primal-dual framework, that does not require coefficient replication. Risk bounds for problems with structured sparsity inducing penalties (including the lasso and group lasso) were obtained by \citep{maurer} using Rademacher complexities.   Sample complexity bounds for model selection in linear regression using the group lasso (with possibly overlapping groups) also exist \citep{nraistats}. The results naturally hold for the standard group lasso \citep{yuanlin}, since non overlapping groups are a special case.  For the non overlapping case, \citep{stojnicblock} characterized the sample complexity of the group lasso, and also gave a semidefinite program to solve the group lasso under a block sparsity setting. 

For logistic regression, \citep{bachlogistic, buneahonest, Mest, plan1bit} and references therein have extensively characterized the sample complexity of identifying the correct model using $\ell_1$ regularized optimization. The authors in \citep{Mest} extended their results to include Generalized Linear Models as well (GLM's). In \citep{plan1bit},  the authors introduced a new framework to solve the classification problem: minimize \footnote{The authors in \citep{plan1bit} write the problem as a maximization. We minimize the negative of the same function} a linear cost function subject to a constraint on the $\ell_1$ norm of the solution.

%%%%%%%%%%%%%%%%%%%%%%%%%%%%%%%%%%%%%%%%%%%%%%%%
%%%%%%%%%%%%%%%%%%%%%%%%%%%%%%%%%%%%%%%%%%%%%%%%
%%%%%%%%%%%%%%%%%%%%%%%%%%%%%%%%%%%%%%%%%%%%%%%%
%%%%%%%%%%%%%%%%%%%%%%%%%%%%%%%%%%%%%%%%%%%%%%%%
%%%%%%%%%%%%%%%%%%%%%%%%%%%%%%%%%%%%%%%%%%%%%%%%
%%%%%%%%%%%%%%%%%%%%%%%%%%%%%%%%%%%%%%%%%%%%%%%%
%%%%%%%%%%%%%%%%%%%%%%%%%%%%%%%%%%%%%%%%%%%%%%%%
%%%%%%%%%%%%%%%%%%%%%%%%%%%%%%%%%%%%%%%%%%%%%%%%
%%%%%%%%%%%%%%%%%%%%%%%%%%%%%%%%%%%%%%%%%%%%%%%%
%%%%%%%%%%%%%%%%%%%%%%%%%%%%%%%%%%%%%%%%%%%%%%%%

\subsection{Organization}
The rest of the paper is organized as follows: in Section \ref{sec:logitsetup}, we formally state our structured sparse feature selection problem and the main results of this paper. Then in Section \ref{sec:sparsity_patterns}, we argue that the regularizer we propose does indeed help in recovering coefficient sparsity patterns that are both within-and across group sparse, even when the groups overlap. In Section \ref{sec:mwidth}, we leverage ideas from \citep{plan1bit} and derive measurement bounds and consistency results for the SOGlasso under a logistic regression setting. We also extend these results to handle data with correlations in their entries.  We perform experiments on real and toy data in Section \ref{sec:expts}, before concluding the paper and mentioning avenues for future research in Section \ref{sec:conc}. 

%%%%%%%%%%%%%%%%%%%%%%%%%%%%%%%%%%%%%%%%%%%%%%%%
%%%%%%%%%%%%%%%%%%%%%%%%%%%%%%%%%%%%%%%%%%%%%%%%
%%%%%%%%%%%%%%%%%%%%%%%%%%%%%%%%%%%%%%%%%%%%%%%%
%%%%%%%%%%%%%%%%%%%%%%%%%%%%%%%%%%%%%%%%%%%%%%%%
%%%%%%%%%%%%%%%%%%%%%%%%%%%%%%%%%%%%%%%%%%%%%%%%
%%%%%%%%%%%%%%%%%%%%%%%%%%%%%%%%%%%%%%%%%%%%%%%%
%%%%%%%%%%%%%%%%%%%%%%%%%%%%%%%%%%%%%%%%%%%%%%%%
%%%%%%%%%%%%%%%%%%%%%%%%%%%%%%%%%%%%%%%%%%%%%%%%
%%%%%%%%%%%%%%%%%%%%%%%%%%%%%%%%%%%%%%%%%%%%%%%%
%%%%%%%%%%%%%%%%%%%%%%%%%%%%%%%%%%%%%%%%%%%%%%%%

\section{Main Results: Classification with Structured Sparsity}
\label{sec:logitsetup}

We now return to the problem that we wish to solve in this paper, and state our main results in a more formal way. Recall that we are interested in recovering a coefficient vector $\vx^\star$, from (corrupted) linear observations of the form $\langle \vphi_i, \vx^\star \rangle$

The coefficient vector of interest is assumed to have a special structure. Specifically, we assume that $\vx^\star \in \mathcal{C} \subset B^p_2$, where $B^p_2$ is the unit euclidean ball in $\R^p$. This motivates the following optimization problem \citep{plan1bit}:
\begin{equation}
\label{optgen}
\widehat{\vx} = \arg \min_{\vx}  \sum_{i = 1}^n - \vy_{i} \left\langle \vphi_{i}, \vx \right\rangle ~\ \textbf{s.t. } ~\ \vx \in \mathcal{C}.
\end{equation}

The function to be optimized has a very natural interpretation: We assume without loss of generality  that the observations are positively correlated \footnote{If the correlation is negative, the signs of $\vy_i$ can be reversed} with the inner products between the features and the coefficient vector. Hence, a natural thing to do would be to maximize the number of ``sign agreements" between $\vy_i$ and $\langle \vphi_i, \vx \rangle$. The objective function in (\ref{optgen}) maximizes the product of the two terms, a linear relaxation of the quantity we wish to optimize.

The statistical accuracy of $\widehat{x}$ can be characterized in terms of the \emph{mean width} of $\mathcal{C}$, which is defined as follows
\begin{definition}
Let $\vg \in \N(0,\mtx{I})$. The mean width of a set $\mathcal{C}$ is defined as
\[
\omega(\mathcal{C}) = \E_{\vg} \left[ \sup_{\vx \in \mathcal{C}-\mathcal{C}} \langle \vx,\vg \rangle \right],
\]
where $\mathcal{C}-\mathcal{C}$ denotes the Minkowski set difference. 
\end{definition}

We now restate a result from \citep{plan1bit}
\begin{theorem}[Corollary 1.2 in \citep{plan1bit}]
\label{thmplan}
Let $\mPhi \in \R^{n \times p}$ be a matrix with i.i.d. standard Gaussian entries, and let $\mathcal{C} \subset B^p_2$. Fix $\vx^\star \in \mathcal{C}$, and assume the observations follow the model (\ref{ymodel}) above. Then, for $\epsilon > 0$, if 
\[
n \geq C \left(\frac{\sigma_f^2}{\epsilon^2}\right) \omega(\mathcal{C})^2,
\]
then with probability at least $1 - 8 \exp(-c(\frac{\epsilon}{\sigma_f})^2 n)$, the solution $\widehat{\vx}$ to the problem (\ref{optgen}) satisfies
\[
\| \widehat{\vx} - \vx^\star \|^2 \leq \epsilon
\]
with $\sigma_f$ defined in (\ref{deflam}).
\end{theorem}

We abuse notation and define the mean width as
\begin{equation}
\label{defmw}
\omega({\mathcal{C}}) = \E_{\vg} \left[ \sup_{\vx \in \mathcal{C}} \langle \vx,\vg \rangle \right].
\end{equation}
The quantity defined above is a constant multiple of that in the original definition for centrally symmetric sets $\mathcal{C}$ \citep{plan1bit}, which will be the case for the remainder of this paper. 

In this paper, we construct a new penalty that produces a \emph{convex} set $\mathcal{C}$ that encourages structured sparsity in the solution of (\ref{optgen}). We show that the resulting optimization can be efficiently solved. We bound the mean width of the set, which yields new bounds for classification with structured sparsity, via Theorem \ref{thmplan}. We state the main results in this section, and defer the proofs to Section \ref{sec:mwidth}

\subsection{A New Penalty for Structured Sparsity}

Assume that the features can be grouped according to similarity into $K$ (possibly overlapping) groups $\G = \{ G_1, G_2, \ldots, G_K \}$ with the largest group being of size $L$ and consider the following definition of structured sparsity.
\begin{definition}
\label{defka}
We say that a vector $\vx$ is $(k, l )$-group sparse if $\vx$ is supported on at most $k \leq K$ groups and at most  $l$ elements in each active group are non zero. 
\end{definition}
Note that $l = 0$ corresponds to $\vx = \vct{0}$.

To encourage such sparsity patterns we define the following penalty. 
Given a group $G \in \G$, we define the set
\[
\W_G = \left\{ \vw \in \R^p : ~\ \vw_i = 0 ~\ \textbf{if} ~\  i \not \in G \right\}.
\]
We can then define
\[
\W(\vx) = \left\{ \vw_{G_1} \in \W_{G_1}, ~\ \vw_{G_2} \in \W_{G_2}, \ldots, \vw_{G_M} \in \W_{G_M} : ~\ \sum_{G \in \G} \vw_G = \vx  \right\}.
\]
That is, each element of $\W(x)$ is a set of vectors, one from each $\W_G$, such that the vectors sum to $\vx$. As shorthand, in the sequel we write $\{ \vw_G \} \in \W(\vx)$ to mean a set of vectors that form an element in $\W(\vx)$

For any $\vx\in \R^p$, define
\begin{equation}
\label{eq:reggen}
h(\vx) \ := \ \inf_{\{ \vw_G \} \in \W(\vx)} \sum_{G \in \G} \left(\alpha_G  \|\vw_G\|_2 + \beta_G  \| \vw_G\|_1 \right) ,
\end{equation}

where the $\alpha_G, \beta_G>0$ are constants that tradeoff the contributions of the $\ell_2$ and the $\ell_1$ norm terms per group, respectively.  The  {\em SOGlasso} is the optimization in (\ref{optgen}) with $h(\vx)$ as defined in (\ref{eq:reggen}) determining the structure of the constraint set $\mathcal{C}$, and hence the form of the solution $\widehat{\vx}$. The $\ell_2$ penalty promotes the selection of only a subset of the groups, and the $\ell_1$ penalty promotes the selection of only a subset of the features within a group.

To keep the exposition simple, we will work with the following definition of $h(\vx)$ in the rest of the paper:
\begin{equation}
\label{eq:reg}
h(\vx) \ := \ \inf_{\{ \vw_G \} \in \W(\vx)} \sum_{G \in \G} \left( \|\vw_G\|_2 + \fll  \| \vw_G\|_1 \right) .
\end{equation}

Note that the value of $\lambda_1 \geq 0$ can be varied to both emphasize or de emphasize the $\ell_1$ penalty. In almost all applications of interest, the value of $l$ will obviously be unknown, and the quantity $\fll$ needs to be tuned via cross validation. However, for the sake of proving our theorems, we will assume that the quantity is known (our goal is to recover vectors that are $(k,l)$- group sparse). This is consistent with other results in the literature where it is assumed that the parameters are known.

\begin{definition}
\label{def:optrep}
We say the set of vectors $\{ \vw_G \} \in \W(\vx)$ is  an optimal representation of $\vx$ if they achieve the $\inf$ in (\ref{eq:reg}). 
\end{definition}
The objective function in (\ref{eq:reg}) is convex and coercive. Hence,  $ \forall \vx$, an optimal representation always exists. 

%%%%%%%%%%%%%%%%%%%%%%%%%%%%%%%%%%%%%%%%%%%%%%%%%
%%%%%%%%%%%%%%%%%%%%%%%%%%%%%%%%%%%%%%%%%%%%%%%%%
%%%%%%%%%%%%%%%%%%%%%%%%%%%%%%%%%%%%%%%%%%%%%%%%%
%%%%%%%%%%%%%%%%%%%%%%%%%%%%%%%%%%%%%%%%%%%%%%%%%
%%%%%%%%%%%%%%%%%%%%%%%%%%%%%%%%%%%%%%%%%%%%%%%%%
%%%%%%%%%%%%%%%%%%%%%%%%%%%%%%%%%%%%%%%%%%%%%%%%%
%%%%%%%%%%%%%%%%%%%%%%%%%%%%%%%%%%%%%%%%%%%%%%%%%
%%%%%%%%%%%%%%%%%%%%%%%%%%%%%%%%%%%%%%%%%%%%%%%%%
%%%%%%%%%%%%%%%%%%%%%%%%%%%%%%%%%%%%%%%%%%%%%%%%%
%%%%%%%%%%%%%%%%%%%%%%%%%%%%%%%%%%%%%%%%%%%%%%%%%
 %%%%%%%%%%%%%%%%%%%%%%%%%%%%%%%%%%%%%%%%%%%%%%%%%%
%%%%%%%%%%%%%%%%%%%%%%%%%%%%%%%%%%%%%%%%%%%%%%%%%%
%%%%%%%%%%%%%%%%%%%%%%%%%%%%%%%%%%%%%%%%%%%%%%%%%%
%%%%%%%%%%%%%%%%%%%%%%%%%%%%%%%%%%%%%%%%%%%%%%%%%%
%%%%%%%%%%%%%%%%%%%%%%%%%%%%%%%%%%%%%%%%%%%%%%%%%%
%%%%%%%%%%%%%%%%%%%%%%%%%%%%%%%%%%%%%%%%%%%%%%%%%%
%%%%%%%%%%%%%%%%%%%%%%%%%%%%%%%%%%%%%%%%%%%%%%%%%%
%%%%%%%%%%%%%%%%%%%%%%%%%%%%%%%%%%%%%%%%%%%%%%%%%%
%%%%%%%%%%%%%%%%%%%%%%%%%%%%%%%%%%%%%%%%%%%%%%%%%%
%%%%%%%%%%%%%%%%%%%%%%%%%%%%%%%%%%%%%%%%%%%%%%%%%%

The function $h(\vx)$ yields a convex relaxation for $(k,l)$-group sparsity. Define the constraint set
\begin{equation}
\label{constraint_set}
\mathcal{C} = \{ \vx: h(\vx) \leq \sqrt{k} \left( 1+  \lambda_1 \right) , ~\ \| \vx \|_2 \leq 1 \} \ .
\end{equation}
We show that $\mathcal{C}$ is convex and contains all $(k,l)$-group sparse vectors.
We compute the mean width of $\mathcal{C}$ in (\ref{constraint_set}), and subsequently obtain the following result:
\begin{theorem}
\label{thmmain}
Suppose there exists a coefficient vector $\vx^\star$ that is $(k, l)$-group sparse. Suppose the data matrix $\mPhi \in \R^{n \times p}$ and observation model follow the setting in Theorem \ref{thmplan}. Suppose we solve (\ref{optgen}) for the constraint set given by (\ref{constraint_set}). For $\epsilon > 0$ and a constant $C$, if the number of measurements satisfies
\[
n \ \geq \ C \frac{\sigma_f^2}{\epsilon^2}k ~\min \left\{ (1 + \lambda_1)^2 (\log{(K)} + L) , \left( \frac{1+\lambda_1}{\lambda_1} \right)^2 l \log{(p)}  \right\},
\]
then with high probability, the solution of the  SOGlasso satisfies
\[
\left\| \widehat{\vx} - \vx^\star \right\|^2_2 \ \leq \epsilon.
\]
\end{theorem}

\subsubsection*{Remarks}
The results of Theorem \ref{thmmain} yield new results for classification with structured sparsity under the general binary observation setting (\ref{ymodel}). Specifically, note that the SOGlasso interpolates between the standard $\ell_1$ regularized (lasso) and the group $\ell_1$ regularized (group lasso) classification techniques:  
\begin{itemize}
\item When $\lambda_1 = 0$, we obtain results for the group lasso. The result remains the same whether or not the groups overlap. The bound is given by
\[
n \geq C \delta^{-2} k (\log(K) + L).
\]
Note that this result is similar to that obtained for the linear regression case by the authors in \citep{nraistats}.
\item When all the groups are singletons, $(L=l=1)$ , the bound reduces to that for the standard lasso, with $KL = p$ being the ambient dimension. In this case, the signal sparsity $s := kl$ and the bound becomes:
\[
n \geq C \delta^{-2} kl \log(p).
\]
\end{itemize}

%%%%%%%%%%%%%%%%%%%%%%%%%%%%%%%%%%%%%%%%%%%%%%%%%%%%
%%%%%%%%%%%%%%%%%%%%%%%%%%%%%%%%%%%%%%%%%%%%%%%%%%%%
%%%%%%%%%%%%%%%%%%%%%%%%%%%%%%%%%%%%%%%%%%%%%%%%%%%%
%%%%%%%%%%%%%%%%%%%%%%%%%%%%%%%%%%%%%%%%%%%%%%%%%%%%
%%%%%%%%%%%%%%%%%%%%%%%%%%%%%%%%%%%%%%%%%%%%%%%%%%%%
%%%%%%%%%%%%%%%%%%%%%%%%%%%%%%%%%%%%%%%%%%%%%%%%%%%%
%%%%%%%%%%%%%%%%%%%%%%%%%%%%%%%%%%%%%%%%%%%%%%%%%%%%
%%%%%%%%%%%%%%%%%%%%%%%%%%%%%%%%%%%%%%%%%%%%%%%%%%%%
%%%%%%%%%%%%%%%%%%%%%%%%%%%%%%%%%%%%%%%%%%%%%%%%%%%%
%%%%%%%%%%%%%%%%%%%%%%%%%%%%%%%%%%%%%%%%%%%%%%%%%%%%
%%%%%%%%%%%%%%%%%%%%%%%%%%%%%%%%%%%%%%%%%%%%%%%%%%%%
%%%%%%%%%%%%%%%%%%%%%%%%%%%%%%%%%%%%%%%%%%%%%%%%%%%%

The SOGlasso generalizes the lasso and the group lasso, and allows one to recover signals that are sparse, group sparse, or a combination of the two structures.

Moreover, since the model we consider subsumes the logistic regression setting, we obtain results for logistic regression with a general structured sparsity constraint on the solution. To the best of our knowledge, these are the first known sample complexity bounds for the group lasso for logistic regression with overlapping groups, and the sparse group lasso, both of which arise as special cases of the SOGlasso. 

Problem (\ref{optgen}) admits an efficient solution. Specifically, we can use the feature replication strategy as in \citep{jacob} to reduce the problem to a sparse group lasso, and use proximal point methods to recover the coefficient vector. We elaborate this in more detail later in the paper. 

Theorem \ref{thmmain} bounds the number of measurements sufficient for accurate estimation using the overlapping group lasso and the lasso. A natural question to then ask is whether one can do better when it is known that the vectors we are interested in are further constrained by the number of non zero entries in active groups. When it is known that each coefficient belongs to at most $R$ groups \footnote{The value of R will always be known, since we assume that the groups $\G$ are known}, we obtain the following sample complexity bound

\begin{theorem}
\label{thmnewbound}
Suppose the coefficient vector is $(k,l)$ group sparse, with everything else the same as in Theorem \ref{thmmain}. Suppose that each coefficient belongs to at most $R$ groups. Then, with high probability, if the number of measurements satisfies
\[
n \geq C \frac{\sigma_f^2}{\epsilon^2}R^2 k \left[ \log{\left( \frac{K}{k} \right)} + l \log{\left( \frac{L}{l} \right)} + l +2 \right],
\]
we have
\[
\| \hat{\vx} - \vx^\star \|^2 \leq \epsilon.
\]
\end{theorem}

The above result yields a tight bound when the groups do not overlap. Indeed, when $R = 1$ we see that the sample complexity bound is a function of the logarithm of the number of groups, and the overall sparsity of the signal $kl$. This is a tighter result that the bound obtained for the group lasso, where the second term would be $\mathcal{O}(kL)$. We pay a price of $R^2$ for overlapping groups, but in most practical examples we are interested in, $R$ is typically small, and is a constant.  

%%%%%%%%%%%%%%%%%%%%%%%%%%%%%%%%%%%%%%%%%%%%%%%%%%%%
%%%%%%%%%%%%%%%%%%%%%%%%%%%%%%%%%%%%%%%%%%%%%%%%%%%%
%%%%%%%%%%%%%%%%%%%%%%%%%%%%%%%%%%%%%%%%%%%%%%%%%%%%
%%%%%%%%%%%%%%%%%%%%%%%%%%%%%%%%%%%%%%%%%%%%%%%%%%%%

\section{Analysis of the SOGlasso Penalty}
\label{sec:sparsity_patterns}

Recall the definition of $h(\vx)$, from (\ref{eq:reg}):
\begin{equation}
\label{eq:regsimple}
h(\vx) = \inf_{\{ \vw_G \} \in \W(\vx)} \sum_{G \in \G} \|\vw_G\|_2 + \mu  \| \vw_G\|_1  
\end{equation}

where we set $\mu = \fll$ 

\subsubsection*{Remarks : }
The SOGlasso penalty can be seen as a generalization of different penalty functions previously explored in the context of sparse linear regression and/or classification:
\begin{itemize}
\item If each group in $\G$ is a singleton, then the SOGlasso penalty reduces to the standard $\ell_1$ norm, and the problem reduces to the lasso  \citep{tibshirani, buneahonest}
\item if $\lambda_1 = 0$ in (\ref{eq:reg}), then we are left with the latent group lasso \citep{jacob, latent, nraistats}. This allows us to recover sparsity patterns that can be expressed as lying in a union of groups. If a group is selected, then all the coefficients in the group are selected.
\item If the groups $G \in \G$ are non overlapping, then (\ref{eq:reg}) reduces to the sparse group lasso \citep{sgl}. Of course, for non overlapping groups, if $\lambda_1 = 0$, then we get the standard group lasso \citep{yuanlin}.
\end{itemize}

Figure \ref{mueffect} shows the effect that the parameter $\mu$ has on the shape of the ``ball" $\| \vw_G \| + \mu \| \vw_G \|_1 \leq \delta$, for a single two dimensional group $G$. 

\begin{figure}[!h]
\centering
\subfigure[$\mu = 0$]{
\includegraphics[width = 25mm, height = 25mm]{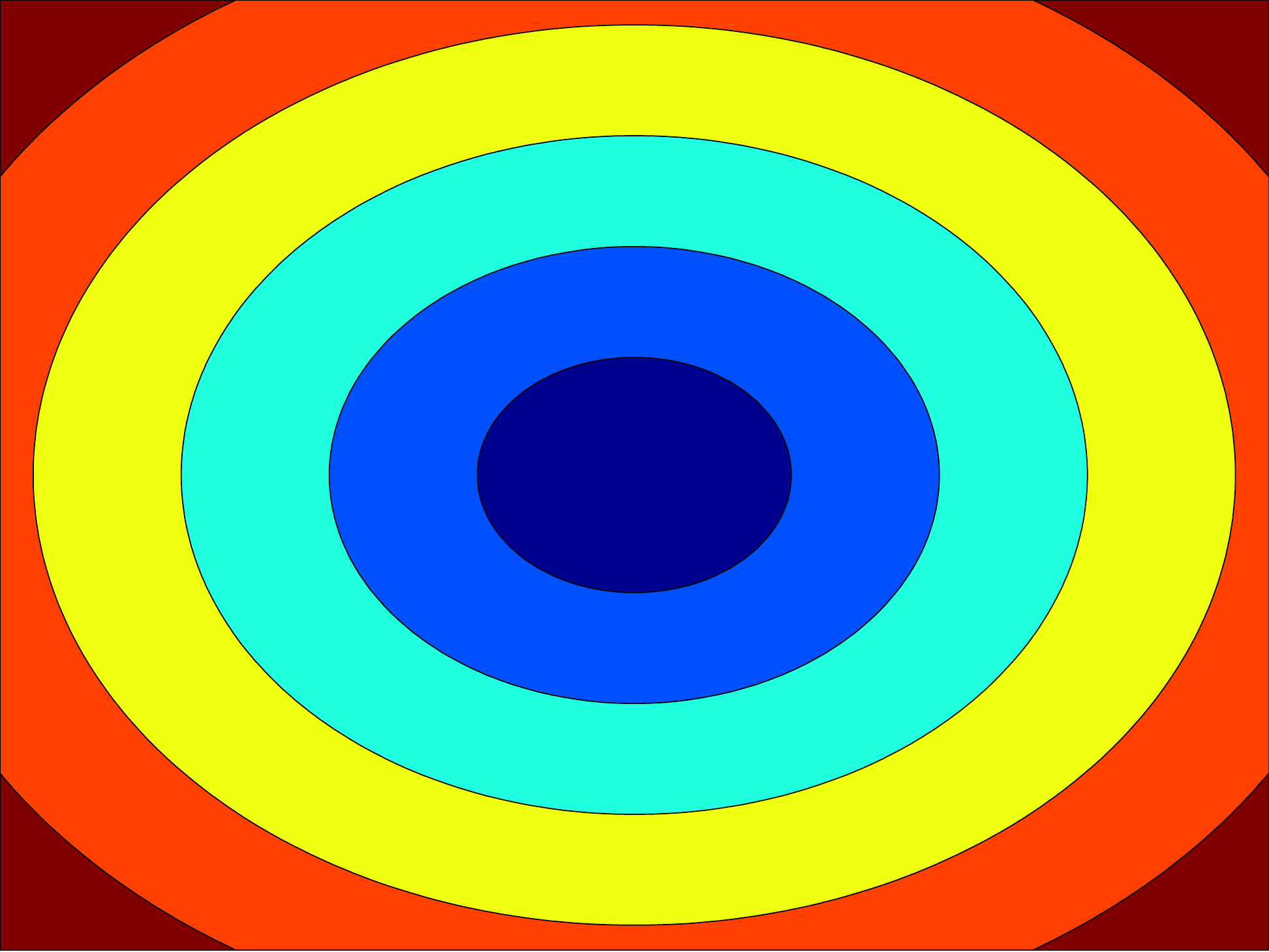}
\label{mu0}}
\qquad
\subfigure[$\mu = 0.2$]{
\includegraphics[width = 25mm, height = 25mm]{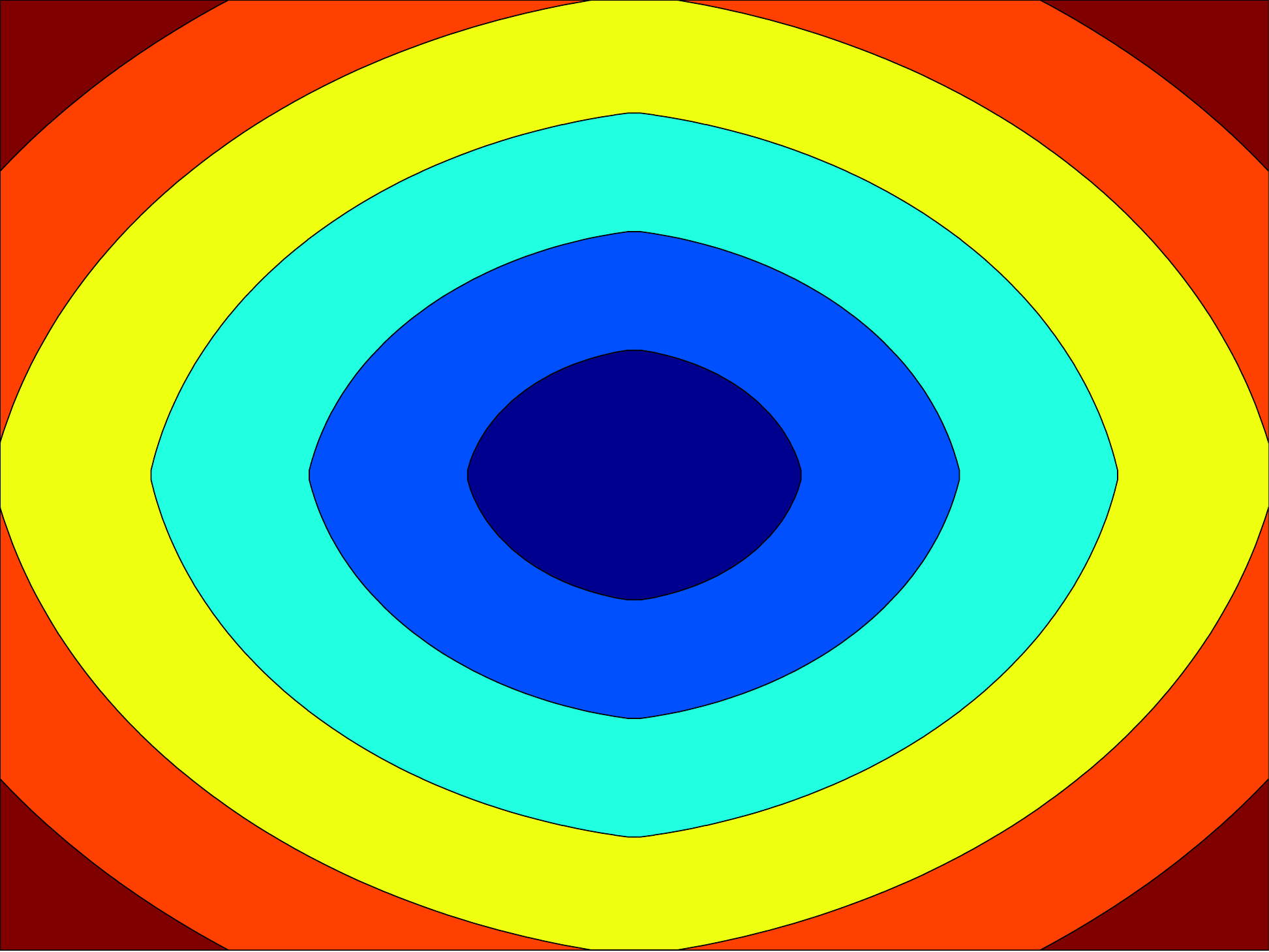}
\label{mu1}}
\qquad
\subfigure[$\mu = 1$ ]{
\includegraphics[width = 25mm, height = 25mm]{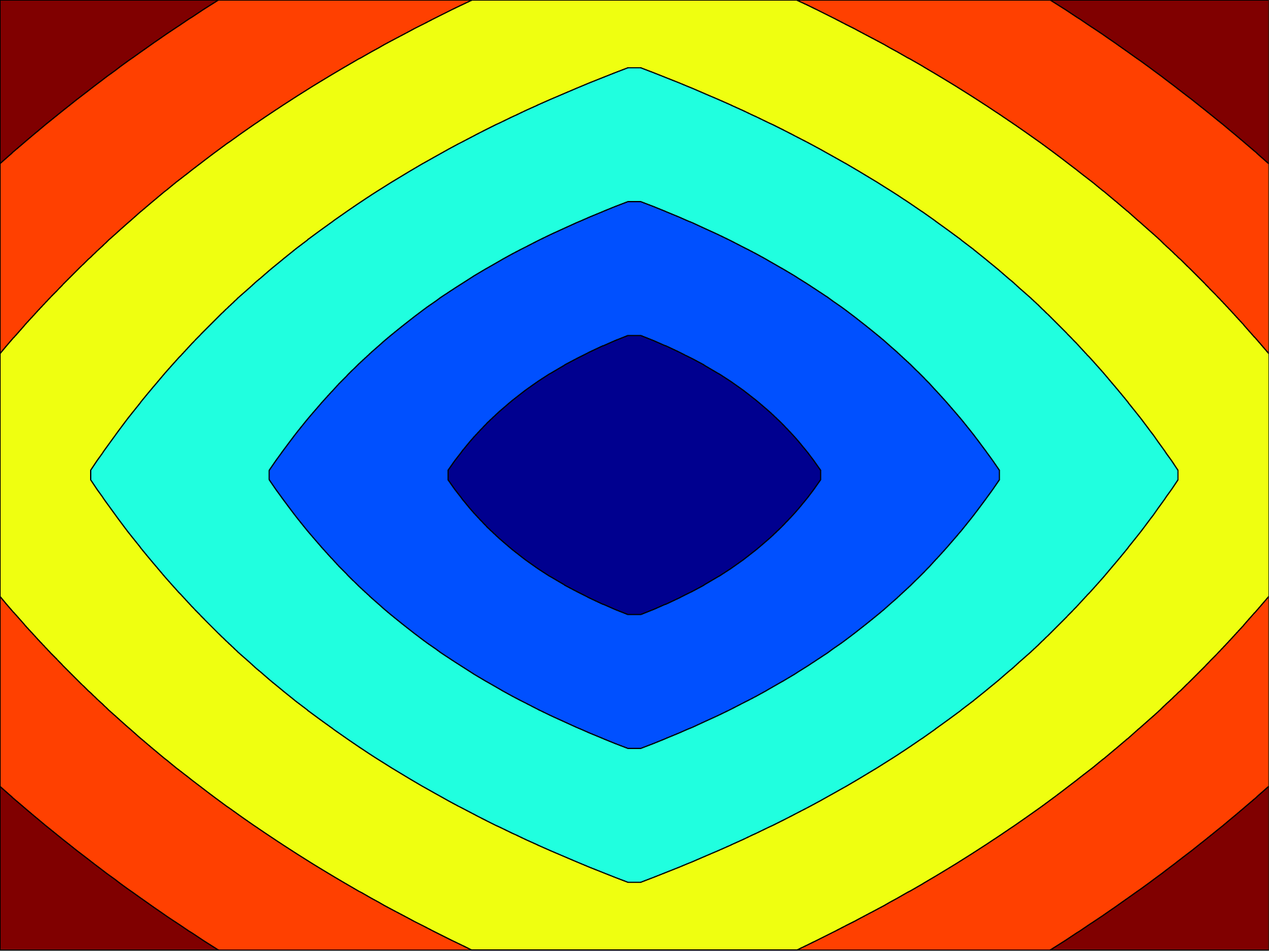}
\label{mu2}}
\qquad
\subfigure[$\mu = 10$ ]{
\includegraphics[width = 25mm, height = 25mm]{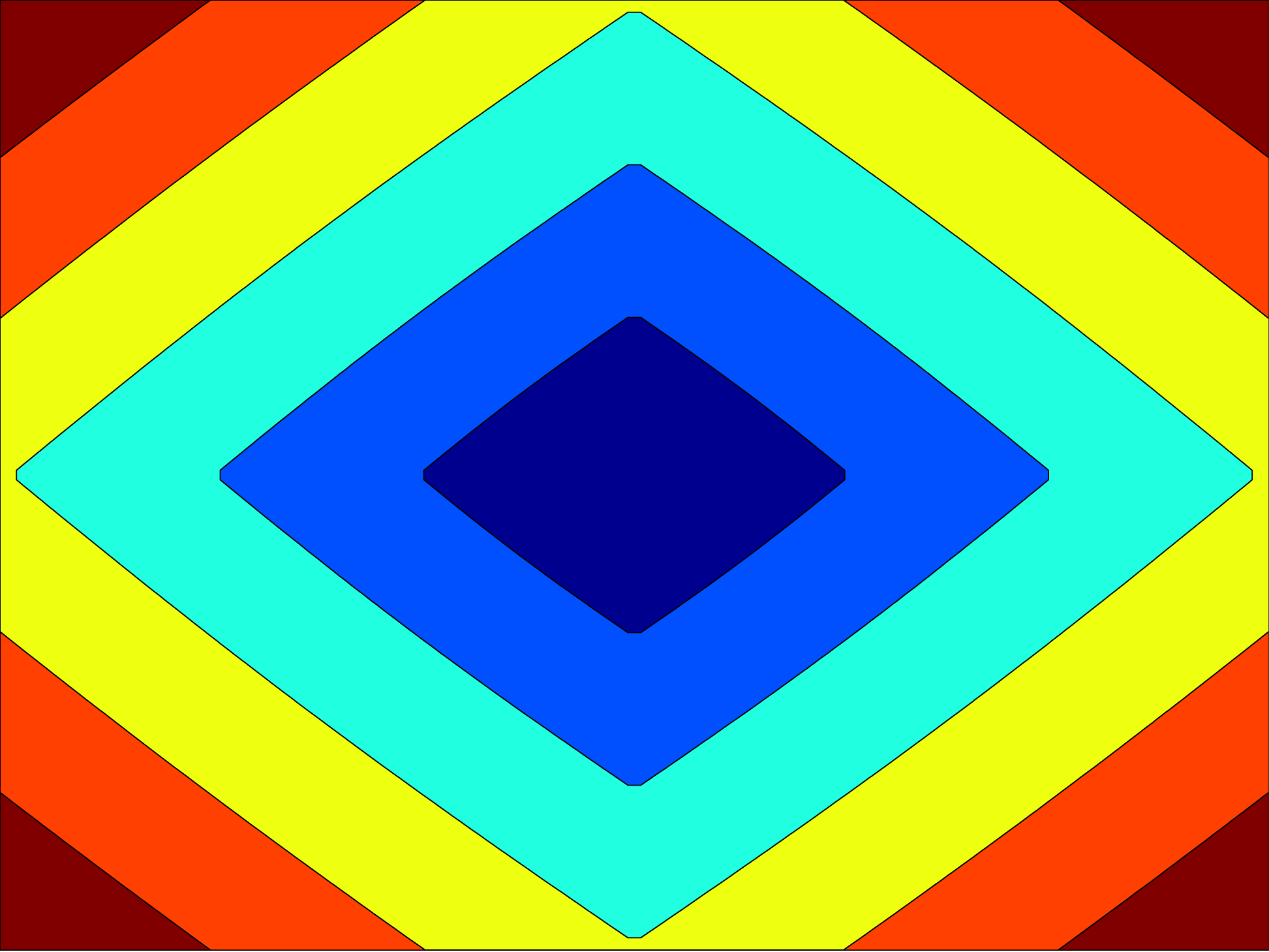}
\label{mu3}}
\caption{ Effect of $\mu$ on the shape of the set $\| \vw_G \| + \mu \| \vw_G \|_1 \leq \delta$, for a two dimensional group $G$. $\mu = 0$ \subref{mu0} yields level sets for the $\ell_2$ norm ball. As the value of $\mu$ in increased, the effect of the $\ell_1$ norm term increases \subref{mu1} \subref{mu2}. Finally as $\mu$ gets very large, the  level sets resemble that of the $\ell_1$ ball \subref{mu3}. }
\label{mueffect}
\end{figure}

\subsection{Properties of SOGlasso Penalty}
The example in Table \ref{tab:which} gives an insight into the kind of sparsity patterns preferred by the function $h(\vx)$. We will tend to prefer solutions that have a small value of $h(\cdot)$. Consider 3 instances of $\vx \in \R^{10}$, and the corresponding group lasso, $\ell_1$ norm, and $h(\vx)$ function values. The vector is assumed to be made up of two groups, $G_1 = \{1,2,3,4,5 \} \mbox{ and } G_2 = \{6,7,8,9,10 \}$. $h(\vx)$ is smallest when the support set is sparse within groups, and also when only one of the two groups is selected (column 5). The $\ell_1$ norm does not take into account sparsity across groups (column 4), while the group lasso norm does not take into account sparsity within groups (column 3). Since the groups do not overlap, the latent group lasso penalty reduces to the group lasso penalty and $h(\vx)$ reduces to the sparse group lasso penalty. 
\begin{table}[!h]
  \centering
\begin{tabular}{ || c | c | c | c | c || }
\hline
\textbf{Support} & \textbf{Values}  & $\sum_{G} \|\vx_G \|$ & $\| \vx \|_1$ & $\sum_{G} \left( \| \vx_G \| + \|\vx_G \|_1 \right)$  \\
  \hline                       
  $\{1,4,9\}$ & $\{3,4,7\}$  &  $12$ & $14$  & $26$ \\
  \hline
  $\{1,2,3,4,5\}$ & $\{2,5,2,4,5\}$  &  $8.602$ & $18$  & $26.602$ \\
  \hline  
  $\{1,3,4\}$ & $\{3,4,7\}$  &  $8.602$ & $14$  & $22.602$ \\
  \hline
\end{tabular}
  \caption{Different instances of a 10-d vector and their corresponding norms.}
 \label{tab:which}
  \end{table}

The next table shows that $h(\vx)$ indeed favors solutions that are not only group sparse, but also exhibit sparsity within groups when the groups overlap. Consider again a 10-dimensional vector $\vx$ with three overlapping groups $\{1,2,3,4 \}$, $\{ 3,4,5,6,7 \}$ and $\{ 7,8,9,10 \}$. Suppose the vector $\vx = [0 ~\ 0 ~\ 1 ~\ 0 ~\ 1 ~\ 0 ~\ 1 ~\ 0 ~\ 0 ~\ 0]^T$. From the form of the function in (\ref{eq:reg}), we see that the vector can be seen as a sum of three vectors $\vw_i, ~\ i = 1,2,3$, corresponding to the three groups listed above. Consider the following instances of the $\vw_i$ vectors, which are all feasible solutions for the optimization problem in (\ref{eq:regsimple}):
\begin{enumerate}
\item $\vw_1 = [0 ~\ 0 ~\ -1 ~\ 0 ~\ 0 ~\ 0 ~\ 0 ~\ 0 ~\ 0 ~\ 0]^T$, $\vw_2 = [0 ~\ 0 ~\ 1 ~\ 1 ~\ 1 ~\ 0 ~\ 1 ~\ 0 ~\ 0 ~\ 0]^T$, $\vw_3 = [0 ~\ 0 ~\ 0 ~\ 0 ~\ 0 ~\ 0 ~\ 0 ~\ 0 ~\ 0 ~\ 0]^T$
\item $\vw_1 = [0 ~\ 0 ~\ 1 ~\ 0 ~\ 0 ~\ 0 ~\ 0 ~\ 0 ~\ 0 ~\ 0]^T$, $\vw_2 = [0 ~\ 0 ~\ 0 ~\ 0 ~\ 1 ~\ 0 ~\ 0 ~\ 0 ~\ 0 ~\ 0]^T$, $\vw_3 = [0 ~\ 0 ~\ 0 ~\ 0 ~\ 0 ~\ 0 ~\ 1 ~\ 0 ~\ 0 ~\ 0]^T$
\item $\vw_1 = [0 ~\ 0 ~\ 0 ~\ 0 ~\ 0 ~\ 0 ~\ 0 ~\ 0 ~\ 0 ~\ 0]^T$, $\vw_2 = [0 ~\ 0 ~\ 1 ~\ 0 ~\ 1 ~\ 0 ~\ 0 ~\ 0 ~\ 0 ~\ 0]^T$, $\vw_3 = [0 ~\ 0 ~\ 0 ~\ 0 ~\ 0 ~\ 0 ~\ 1 ~\ 0 ~\ 0 ~\ 0]^T$
\item $\vw_1 = [0 ~\ 0 ~\ 0 ~\ 0 ~\ 0 ~\ 0 ~\ 0 ~\ 0 ~\ 0 ~\ 0]^T$, $\vw_2 = [0 ~\ 0 ~\ 1 ~\ 0 ~\ 1 ~\ 0 ~\ 1 ~\ 0 ~\ 0 ~\ 0]^T$, $\vw_3 = [0 ~\ 0 ~\ 0 ~\ 0 ~\ 0 ~\ 0 ~\ 0 ~\ 0 ~\ 0 ~\ 0]^T$
\end{enumerate}

In the above list, the first instance corresponds to the case where the support is localized to two groups, and one of these groups (group 2) has only one zero. The second case corresponds to the case where all 3 groups have non zeros in them. The third case has support localized to two groups, and both groups are sparse. Finally, the fourth case has only the second group having non zero coefficients, and this group is also sparse. Table \ref{tab:sparse} shows that the smallest value of the sum of the terms is achieved by the fourth decomposition, and hence that will correspond to the optimal representation.

\begin{table}[!h]
  \centering
\begin{tabular}{ || c | c | c | c  || }
\hline
$A = \| \vw_1 \| + \mu \| \vw_1 \|_1 $& $B = \| \vw_2 \| + \mu \| \vw_2 \|_1 $  & $C = \| \vw_3 \| + \mu \| \vw_3 \|_1 $ & $A+B+C$  \\
 \hline                       
   $1+\mu$ & $2+4 \mu$  &  $0$ &  $3 + 5 \mu$ \\
  \hline                       
   $1+\mu$ & $1+\mu$  &  $1+\mu$ &  $3 + 3 \mu$ \\
  \hline
  $0$ & $\sqrt{2}+2\mu$  &  $1+\mu$ &  $1 + \sqrt{2} + 3 \mu$ \\
  \hline  
 $0$ & $\sqrt{3}+3\mu$  &  $0$ &  $\sqrt{3} + 3 \mu$ \\
  \hline
\end{tabular}
  \caption{Values of the sum of the $\ell_1$ and $\ell_2$ norms corresponding to the decompositions listed above. Note that the optimal representation corresponds to the case $\vw_1 = \vw_3 = \vct{0}$, and $\vw_2$ being a sparse vector.}
 \label{tab:sparse}
  \end{table}

Lastly, we can show that $h(\vx)$ is a norm.  This will imply that $h(\vx)$ is convex, and hence the penalty we consider will be convex. This will then mean that the optimization we are interested in solving is a convex program.  

\begin{lemma}
\label{isnorm}
The function
\[
h(\vx) = \inf_{\{ \vw_G \} \in \W(\vx)} \sum_{G \in \G} \left( \|\vw_G\|_2 + \mu  \| \vw_G\|_1 \right)
\] 
is a norm
\end{lemma}
\begin{proof}
It is trivial to show that $h(\vx) \geq 0$ with equality \emph{iff} $\vx = 0$. We now show positive homogeneity. 
Suppose $\{ \vw_G \} \in \W(\vx)$ is an optimal representation (Definition \ref{def:optrep}) of $\vx$, and let $\gamma \in \R \backslash \{\vct{0} \}$. Then, $\sum_{G \in \G} \vw_G = \vx ~\ \Rightarrow \sum_{G \in \G} \gamma \vw_G = \gamma \vx$. This leads to the following set of inequalities:
\begin{align}
h(\vx) &= \sum_{G \in \G} \left( \| \vw_G \|_2 + \mu \| \vw_G  \|_1 \right) 
= \frac{1}{|\gamma|} \sum_{G \in \G} \left( \| \gamma \vw_G \|_2 + \mu \| \gamma \vw_G  \|_1 \right) 
\label{oneside}
 \geq \frac{1}{|\gamma|} h(\gamma \vx)
\end{align}
Now, assuming $\{ \vv_G \} \in \W(\gamma \vx)$ is an optimal representation of $\gamma \vx$, we have that $\sum_{G \in \G}  \frac{\vv_G}{\gamma} = \vx$, and we get
\begin{align}
h(\gamma \vx) &= \sum_{G \in \G} \left( \| \vv_G\|_2 +\mu \| \vv_G \|_1  \right) 
= |\gamma| \sum_{G \in \G} \left( \left\| \frac{\vv_G}{\gamma} \right\|_2 + \mu \left\| \frac{\vv_G}{\gamma} \right\|_1  \right) 
\label{otherside}
\geq |\gamma| h(\vx)
\end{align}
Positive homogeneity follows from (\ref{oneside}) and (\ref{otherside}). The inequalities are a result of the possibility of the vectors not corresponding to the respective optimal representations. 

For the triangle inequality, again let $\{\vw_G \} \in \W(\vx),\{ \vv_G \} \in \W(\vy)$ correspond to the optimal representation  for $\vx , \vy$ respectively. Then by definition, 
\begin{align*}
h(\vx + \vy) &\leq \sum_{G \in \G} (\|\vw_G + \vv_G\|_2 + \mu \|\vw_G  +  \vv_G \|_1) \\
&\leq \sum_{G \in \G} (\|\vw_G \|_2 + \|\vv_G\|_2 + \mu \|\vw_G\|_1  +\mu  \| \vv_G \|_1) \\
&= h(\vx) + h(\vy)
\end{align*}
The first and second inequalities follow by definition and the triangle inequality respectively.
\end{proof}

%%%%%%%%%%%%%%%%%%%%%%%%%%%%%%%%%%%%%%%%%%%%%%%%%%%%
%%%%%%%%%%%%%%%%%%%%%%%%%%%%%%%%%%%%%%%%%%%%%%%%%%%%
%%%%%%%%%%%%%%%%%%%%%%%%%%%%%%%%%%%%%%%%%%%%%%%%%%%%
%%%%%%%%%%%%%%%%%%%%%%%%%%%%%%%%%%%%%%%%%%%%%%%%%%%%
%%%%%%%%%%%%%%%%%%%%%%%%%%%%%%%%%%%%%%%%%%%%%%%%%%%%
%%%%%%%%%%%%%%%%%%%%%%%%%%%%%%%%%%%%%%%%%%%%%%%%%%%%
%%%%%%%%%%%%%%%%%%%%%%%%%%%%%%%%%%%%%%%%%%%%%%%%%%%%
%%%%%%%%%%%%%%%%%%%%%%%%%%%%%%%%%%%%%%%%%%%%%%%%%%%%

\subsection{Solving the SOGlasso Problem}
 We solve the Lagrangian version of the SOGlasso problem:
\[
\hat{\vx} = \arg \min_{\vx} \left( \sum_{i = 1}^n - \vy_{i} \left\langle \vphi_{i}, \vx \right\rangle \right) + \eta_1 h(\vx) + \eta_2 \| \vx \|^2,
\] 
where $\eta_1 > 0$ controls the amount by which we regularize the coefficients to have a structured sparsity pattern, and $\eta_2 > 0$ prevents the coefficients from taking very large values. We use the ``covariate duplication" method of \citep{jacob} to first reduce the problem to the non overlapping sparse group lasso in an expanded space.  One can then use proximal methods to recover the coefficients. 

Proximal point methods progress by taking a step in the direction of the negative gradient, and applying a shrinkage/proximal point mapping to the iterate. This mapping can be computed efficiently for the non overlapping sparse group lasso, as it is a special case of general hierarchical structured penalties \citep{bachhierarchical}. The proximal point mapping can be seen as the composition of the standard soft thresholding and the group soft thresholding operators:
\begin{align*}
\tilde{\vw} &= \mbox{sign}(\vw_{\nabla}) \left[ |\vw_{\nabla} | - \eta_1 \mu \right]^+ \\
(\vw_{t+1})_G &= \frac{(\tilde{\vw})_G}{\| (\tilde{\vw})_G  \|} \left[ \| (\tilde{\vw}_G  \| - \eta_1 \right]^+ ~\ \textbf{if }  \| (\tilde{\vw})_G  \| \neq 0 \\
(\vw_{t+1})_G &= 0 ~\ \textbf{otherwise}
\end{align*}
where $\vw_{\nabla}$ corresponds to the iterate after a gradient step and $[\cdot]^+ = \max \left( 0, \cdot \right)$. 
Once the solution is obtained in the duplicated space, we then recombine the duplicates to obtain the solution in the original space. Finally, we perform a debiasing step to obtain the final solution. 

%%%%%%%%%%%%%%%%%%%%%%%%%%%%%%%%%%%%%%%%%%%%%%%%%%%%
%%%%%%%%%%%%%%%%%%%%%%%%%%%%%%%%%%%%%%%%%%%%%%%%%%%%
%%%%%%%%%%%%%%%%%%%%%%%%%%%%%%%%%%%%%%%%%%%%%%%%%%%%
%%%%%%%%%%%%%%%%%%%%%%%%%%%%%%%%%%%%%%%%%%%%%%%%%%%%
%%%%%%%%%%%%%%%%%%%%%%%%%%%%%%%%%%%%%%%%%%%%%%%%%%%%
%%%%%%%%%%%%%%%%%%%%%%%%%%%%%%%%%%%%%%%%%%%%%%%%%%%%
%%%%%%%%%%%%%%%%%%%%%%%%%%%%%%%%%%%%%%%%%%%%%%%%%%%%
%%%%%%%%%%%%%%%%%%%%%%%%%%%%%%%%%%%%%%%%%%%%%%%%%%%%

\section{Proof of Theorem~\ref{thmmain}, Theorem ~\ref{thmnewbound} and Extensions to Correlated Data}
\label{sec:mwidth}

%%%%%%%%%%%%%%%%%%%%%%%%%%%%%%%%%%%%%%%%%%%%%%%%%%%%
%%%%%%%%%%%%%%%%%%%%%%%%%%%%%%%%%%%%%%%%%%%%%%%%%%%%
%%%%%%%%%%%%%%%%%%%%%%%%%%%%%%%%%%%%%%%%%%%%%%%%%%%%
%%%%%%%%%%%%%%%%%%%%%%%%%%%%%%%%%%%%%%%%%%%%%%%%%%%%%%%%%%%%%%%%%%%%%%%%%%%%%%%%%%%%%%%%%%%%%%%%%%%%%%%%
%%%%%%%%%%%%%%%%%%%%%%%%%%%%%%%%%%%%%%%%%%%%%%%%%%%%
%%%%%%%%%%%%%%%%%%%%%%%%%%%%%%%%%%%%%%%%%%%%%%%%%%%%
%%%%%%%%%%%%%%%%%%%%%%%%%%%%%%%%%%%%%%%%%%%%%%%%%%%%
%%%%%%%%%%%%%%%%%%%%%%%%%%%%%%%%%%%%%%%%%%%%%%%%%%%%
%%%%%%%%%%%%%%%%%%%%%%%%%%%%%%%%%%%%%%%%%%%%%%%%%%%%
%%%%%%%%%%%%%%%%%%%%%%%%%%%%%%%%%%%%%%%%%%%%%%%%%%%%
%%%%%%%%%%%%%%%%%%%%%%%%%%%%%%%%%%%%%%%%%%%%%%%%%%%%%%%%%%%%%%%%%%%%%%%%%%%%%%%%%%%%%%%%%%%%%%%%%%%%%%%%
%%%%%%%%%%%%%%%%%%%%%%%%%%%%%%%%%%%%%%%%%%%%%%%%%%%%
%%%%%%%%%%%%%%%%%%%%%%%%%%%%%%%%%%%%%%%%%%%%%%%%%%%%
%%%%%%%%%%%%%%%%%%%%%%%%%%%%%%%%%%%%%%%%%%%%%%%%%%%%

In this section, we compute the mean width of the constraint set $\mathcal{C}$ in (\ref{constraint_set}), which will be used to prove  Theorems \ref{thmmain} and \ref{thmnewbound}. First we define the following function, analogous to the $\ell_0$ pseudo-norm:
\begin{definition}
\label{defgzero}
Given a set of $K$ groups $\G$, for any vector $\vx$ and its optimal representation $\{ \vw_G \} \in \W(\vx)$, noting that $x = \sum_{G \in \G} \vw_G$, define 
\[
\| \vx\|_{\G,0} = \sum_{G \in \G} \mathds{1}_{\{ \| \vw_{G} \| \neq 0 \}} .
\]
\end{definition}
In the above definition, $\mathds{1}_{\{ \cdot \} }$ is the indicator function. 

Define the set
\begin{equation}
\label{defcideal}
\mathcal{C}_{nc}(k, l) = \left\{ \vx : \vx = \sum_{G \in \G} \vw_G,  ~\ \| \vx \|_{\G,0} \leq k, ~\ \sum_{G \in \G} \| \vw_G \|_0 \leq kl ~\ \forall G \in \G \right\}.
\end{equation}

We see that   $\mathcal{C}_{nc}(k, l)$ contains $(k,l)$-group sparse signals (Definition \ref{defka}). From the above definitions and our problem setup, our aim is to ideally solve the following optimization problem
\begin{equation}
\label{optideal}
\widehat{\vx} = \arg \min_{\vx}  \sum_{i = 1}^n - \vy_{ti} \left\langle \vphi_{ti}, \vx_t \right\rangle ~\ \textbf{s.t. } ~\ \vx \in \mathcal{C}_{nc}(k, l)
\end{equation}
However, the set $\mathcal{C}_{nc}(k,l)$ is not convex, and hence solving (\ref{optideal}) will be hard in general. We instead consider a convex relaxation of the above problem. The convex relaxation of the (overlapping) group $\ell_0$ pseudo-norm is the (overlapping) group $\ell_1/ \ell_2$ norm. This leads to the following result:

\begin{lemma}
\label{lemconvrelax}
The SOGlasso penalty (\ref{eq:regsimple}) admits a convex relaxation of $\mathcal{C}_{ideal}(k, \alpha)$. Specifically, we can consider the set
\[
\mathcal{C}(k, l)  = \{ \vx : h(\vx) \leq \sqrt{k} (1 + \lambda_1) \| \vx \|_2 \}
\]
as a tight convex relaxation containing the set $\mathcal{C}_{nc}(k, l)$.
\end{lemma}
\begin{proof}
Consider a $(k,l)$-group sparse vector $\vx \in \mathcal{C}_{nc}(k, l)$. For any such vector, there exist vectors $ \{ \vv_G \} \in \W(\vx)$ such that the supports of $\vv_G$ do not overlap. We then have the following set of inequalities
\begin{align*}
h(\vx) & = \inf_{\{ \vw_G \} \in \W(\vx)} \sum_{G \in \G} \|\vw_G\|_2 + \fll  \| \vw_G\|_1  \\
&\stackrel{(i)}{\leq} \sum_{G \in \G} \| \vv_G \|_2 + \fll \sum_{G \in \G} \| \vv_G \|_1 \\
&\stackrel{(ii)}{\leq} \sum_{G \in \G} \| \vv_G \|_2 + \fll \sqrt{l} \sum_{G \in \G} \| \vv_G \|_2 \\
&= \left( 1 + \lambda_1 \right) \sum_{G \in \G} \| \vv_G \|_2 \\
&\stackrel{(iii)}{\leq} \sqrt{k}  \left( 1 + \lambda_1 \right) \left( \sum_{G \in \G} \| \vv_G \|^2_2 \right)^\frac12 \\
&= \sqrt{k}  \left( 1 + \lambda_1 \right) \| \vx \|_2
\end{align*}
where (i) follows from the definition of the function $h(\vx)$ in (\ref{eq:reg}), and (ii) and (iii) follow from the fact that for any vector ${\bf v}\in \R^d$ we have $\|{\bf v}\|_1 \leq \sqrt{d}\, \|{\bf v}\|_2$.This, coupled with the fact that $h(\vx)$ is a norm (Lemma \ref{isnorm}) ensures that the set $\mathcal{C}(k,l)$ is convex. 

To show that the relaxation is tight, we will consider a $(k,l)$ sparse vector $\vx$ and show that the inequality in the definition of the set holds with equality. Specifically, let $\vx \in \R^p$ with non overlapping groups, and let the first $k$ $\vw_G$s in it's representation be active. Moreover, suppose the first $l$ entries in each of these $\vw_G$s  are non zero. Let the non zero entries all be equal to $\frac{1}{\sqrt{kl}}$. Then $\| \vx \| = 1$, $\sum_G \| \vw_G \|_2 = \sqrt{k}$ and $\sum_{G} \| \vw_G \|_1 = \sqrt{kl}$. The result follows.
\end{proof}

\subsection{Mean Widths for the SOGlasso }
We see that, the mean width of the constraint set plays a crucial role in determining the consistency of the solution of the optimization problem. We now aim to find the mean width of the constraint set in (\ref{constraint_set}), and as a result of it, prove Theorems \ref{thmmain} and \ref{thmnewbound}. Before we do so, we restate Lemma 3.2 in \citep{nraistats} for the sake of completeness:
\begin{lemma}
\label{lem:chisq}
Let $q_1,\ldots,q_K$ be $K$, $\chi$-squared random variables with $d$-degrees of freedom.  Then \[
\E[\max_{1\leq i \leq K} q_i]\leq (\sqrt{2\log(K)} + \sqrt{d})^2.
\]
\end{lemma}

 \vspace{2mm}
 
 First, we prove Theorem \ref{thmnewbound}. 
 
 \begin{lemma}
 \label{lem:newbound}
 Suppose that it is known that each coefficient is part of at most R groups, and suppose we let
\[
h(\vx) = \inf_{ \{ \vw_{\vx} \} \in \W} \sum_{G in \G} \| \vw_G \|_2 + \fll \| \vw_G \|_1
\] 
 Then the mean width of the set 
 \[
 \mathcal{C} = \{ \vx: h(\vx) \leq \sqrt{k} \left( 1+  \lambda_1 \right) , ~\ \| \vx \|_2 \leq 1 \}
 \]
 is bounded as 
 \[
 \omega({\mathcal{C}})^2 \leq C R^2 k \left[ \log{\left( \frac{K}{k} \right)} + l \log{\left( \frac{L}{l} \right)} + l +2 \right] .
 \]
 \end{lemma}
 \begin{proof}
 The intuition behind this proof is as follows: We first consider a non convex set, which is the ``ideal" set of $(k,l)$ sparse vectors that we are interested in. We then show that $\mathcal{C}$ is contained in the scaled convex hull of the non convex set, and hence by the properties of the mean width, $\omega({\mathcal{C}})$ can be bounded by a scaling of that of the non convex set \footnote{ Lemma \ref{lemconvrelax} showed that the set $\mathcal{C}$ is a tight convex outer relaxation of the non convex set.  }. 
 
To this end, let us consider the following non convex ideal set, 
\begin{equation}
\label{nonconvset}
\mathcal{C}_{nc} = \{ \vx: \|  \vx \| \leq 1, \| \vx \|_{\G,0} \leq k, \sum_{G \in \G} \| \vw_G \|_0 \leq k l \}.
\end{equation}

Consider $\vx \in \mathcal{C}$. We now define vectors $\vx_i$ as follows: $\vx_1 = \sum_{r = 1}^k  \vw_1^r $, where the vectors $\vw_1$ are the $k$ vectors $\vw_G$ with largest norm. Along these lines, we define $\vx_i = \sum_{r = 1}^k \vw_i^r$. 

For a fixed $i$, let $\vx_{i1}$ be the vector containing the top $k l$ entries of $\vx_i$ by magnitude, and define a general $\vx_{ij}$ in this manner as well. 

Note that $\vx_i = \sum_j \vx_{ij}$ and $\vx = \sum_i \vx_i$. Also, note that $\frac{\vx_{ij}}{\| \vx_{ij} \|} \in \mathcal{C}_{nc}$ since it has at most $k$ active groups, and at most $k l$ non zero elements. 

Finally, note the following: By construction, we have for a fixed $i$, and $j > 1$
\begin{equation}
\label{bound1}
\| \vx_{ij} \|_2 \leq \frac{1}{\sqrt{k l}} \| \vx_{ij-1} \|_1.
\end{equation}
This follows since each element of $ \vx_{ij}$ is smaller than the average of the entries of the vector $\vx_{ij-1}$. 

Using the exact same argument, we also have for $i > 1$
\begin{equation}
\label{bound2}
\left( \sum_{r = 1}^k \| \vw_i^r \|^2_2 \right)^{\frac12} \leq \frac{1}{\sqrt{k}} \sum_{r = 1}^k \| \vw_{i-1}^r \|_2.
\end{equation}

Now, 
\begin{align}
\notag
\sum_{ij} \| \vx_{ij} \| &= \| \vx_{11} \|_2 + \sum_{i > 1} \| \vx_{i1} \|_2 + \sum_{i} \sum_{j > 1} \| \vx_{ij} \|_2 \\
\notag
& \leq \| \vx_{11} \|_2 + \sum_{i > 1} \| \vx_{i1} \|_2 + \sum_{i} \sum_{j > 1} \frac{1}{\sqrt{k l}} \| \vx_{ij-1} \|_1 ~\ \mbox{from (\ref{bound1})}  \\
\notag
&\leq \| \vx_{11} \|_2 + \sum_{i > 1} \| \vx_{i1} \|_2 + \sum_{i} \sum_{j} \frac{1}{\sqrt{k l}} \| \vx_{ij} \|_1 \\
\notag
&\leq \| \vx_{11} \|_2 + \sum_{i > 1} \| \vx_{i1} \|_2 + \frac{1}{\sqrt{k l}} \sum_{i} \| \vx_{i} \|_1 ~\ \mbox{since the indices for $j$ are disjoint} \\
\notag
&= \| \vx_{11} \|_2 + \sum_{i > 1} \| \vx_{i1} \|_2 + \frac{1}{\sqrt{k l}} \sum_{i} \left\| \sum_{r =1}^k \vw_i^r  \right\|_1 \\
\notag
&\leq \| \vx_{11} \|_2 + \sum_{i > 1} \| \vx_{i1} \|_2 + \frac{1}{\sqrt{k l}} \sum_{i}  \sum_{r =1}^k \| \vw_i^r \|_1 ~\ \mbox{triangle inequality} \\
\label{leftoff}
&= \| \vx_{11} \|_2 + \sum_{i > 1} \| \vx_{i1} \|_2 + \frac{1}{\sqrt{k l}} \sum_{G}  \| \vw_G \|_1
\end{align}

For $i > 1$, we have the following bound:

\begin{align}
\notag
\| \vx_{i1} \|^2 &\leq \left\| \sum_{r = 1}^k \vw_i^r  \right\|^2 \\
\notag
&= \left[ \sum_{m = 1} ^p \left( \sum_{r = 1}^k (\vw_i^r)_m \right)^2 \right] \\
\notag
&\leq \left[ \sum_{m = 1}^p R^2 \max_{r} (\vw_i^r)^2_m \right] \\
\notag
&\leq R^2 \sum_{m = 1}^p \sum_{r = 1}^k (\vw_i^r)^2_m \\
\notag
&= R^2 \sum_{r = 1}^k \| \vw_i^r \|_2^2
\end{align}

the above inequality and (\ref{bound2}) combine to give
\begin{equation}
\label{bound4}
\| \vx_{i1} \|_2 \leq \frac{R}{\sqrt{k}} \sum_{r = 1}^k \| \vw_{i-1}^r \|_2.
\end{equation}

Substituting this in (\ref{leftoff}) and noting that $\| \vx_{11} \|_2 \leq \| \vx \|_2 \leq 1$, we have
\begin{align}
\sum_i \sum_j \| \vx_{ij} \|_2 &\leq 1 + \frac{1}{\sqrt{k}} \left( R \sum_{G} \| \vw_G \|_2 + \frac{1}{\sqrt{l}} \sum_G \| \vw_G \|_1 \right) \\
&= 1 + \frac{R}{\sqrt{k}} \left(  \sum_{G} \| \vw_G \|_2 + \frac{1}{R \sqrt{l}} \sum_G \| \vw_G \|_1 \right)
\end{align}

If $\lambda_1 \geq \frac{1}{R}$ then the term in the parentheses is bounded by $\sqrt{k} (1 + \fla)$, and if not, then it is bounded by $\sqrt{k} (1 + \frac{1}{R})$. This gives :

\[
\sum_i \sum_j \| \vx_{ij} \|_2 \leq 1 + R + \max{(1, R\lambda_1)} .
\]

The following argument finishes the proof. Setting $\eta =  1 + R + \max{(1, R\lambda_1)}$
\begin{enumerate}
\item By construction, we have $\vx = \sum_i \sum_j \vx_{{ij}}$. 
\item Also by construction, $\frac{\vx_{{ij}}}{\| \vx_{{ij}} \|_2} \in \mathcal{C}_{nc}$.
\item Now, letting $\lambda_{ij} = \frac{\| \vx_{{ij}} \|_2 }{\eta}$, we showed
\[
\frac{\vx}{\eta} = \sum_i \sum_j \lambda_{ij} \frac{ \vx_{{ij}} }{\| \vx_{{ij}} \|_2}
\]
\item We showed that $\sum_i \sum_j \lambda_{ij} \leq 1$, so that $\vx$ can be written as a convex combination of the $\frac{\vx_{{ij}}}{\| \vx_{{ij}} \|_2}$, which are elements in $\mathcal{C}_{nc}$. This means that $\frac{\vx}{\eta} \in conv(\mathcal{C}_{nc})$.
\end{enumerate}

We then have the following bound for the mean width of $\mathcal{C}$:
\[
\omega(\mathcal{C})^2 \leq C R^2  \omega(\mathcal{C}_{nc})^2. 
\]

It now remains to compute $\omega(\mathcal{C}_{nc})$. Lemma \ref{mwnonconvset}  yields the desired result. 

 \end{proof}
 
 \begin{lemma}
 \label{mwnonconvset}
 For
 \[
 \mathcal{C}_{nc} = \{ \vx: \|  \vx \| \leq 1, \| \vx \|_{\G,0} \leq k, \sum_{G \in \G} \| \vw_G \|_0 \leq k l \} , 
 \]
 we have
 \[
 \omega(\mathcal{C}_{nc})^2 \leq C k \left[ \log{\left( \frac{K}{k} \right)} + l \log{\left( \frac{L}{l} \right)} + l +2 \right] . 
 \]
  \end{lemma}
  We prove this  in Appendix \ref{app:mwnonconv}. We make use of Lemma \ref{lem:chisq} to obtain the result. 
  
  %%%%%%%%%%%%%%%%%%%%%%%%%%%%%%%%%%%%%%%%%%%%%%%%
  %%%%%%%%%%%%%%%%%%%%%%%%%%%%%%%%%%%%%%%%%%%%%%%%
  %%%%%%%%%%%%%%%%%%%%%%%%%%%%%%%%%%%%%%%%%%%%%%%%
  %%%%%%%%%%%%%%%%%%%%%%%%%%%%%%%%%%%%%%%%%%%%%%%%
  %%%%%%%%%%%%%%%%%%%%%%%%%%%%%%%%%%%%%%%%%%%%%%%%
  %%%%%%%%%%%%%%%%%%%%%%%%%%%%%%%%%%%%%%%%%%%%%%%%
  %%%%%%%%%%%%%%%%%%%%%%%%%%%%%%%%%%%%%%%%%%%%%%%%
  %%%%%%%%%%%%%%%%%%%%%%%%%%%%%%%%%%%%%%%%%%%%%%%%
  %%%%%%%%%%%%%%%%%%%%%%%%%%%%%%%%%%%%%%%%%%%%%%%%
  %%%%%%%%%%%%%%%%%%%%%%%%%%%%%%%%%%%%%%%%%%%%%%%%
  %%%%%%%%%%%%%%%%%%%%%%%%%%%%%%%%%%%%%%%%%%%%%%%%

We now proceed to prove Theorem \ref{thmmain}. To do so, we adopt a different strategy than the one used to prove Theorem \ref{thmnewbound}. Instead of considering the non convex ideal set, we directly consider the convex set $\mathcal{C}$ and show that it is a subset of appropriately scaled versions of the overlapping group lasso or the lasso balls. The result then follows.

\begin{lemma}
\label{lem:mw}
Consider the same set as that considered in Lemma \ref{lem:newbound}. The mean width of the set can also be shown to satisfy:
\[
\omega(\mathcal{C})^2 \leq C k \min\{\log{K} + L, l \log{(p)}\} . 
\]
\end{lemma}
\begin{proof}
Let $\vg \sim \N(\vct{0} , \mtx{I})$, and for a given $\vx$, let $\{ \vw_G \} \in \W(\vx)$ be its optimal representation (Definition \ref{def:optrep}). Since $\vx = \sum_{G \in \G} \vw_G$, we have  

\begin{align}
\notag
\max_{\vx \in \mathcal{C}} \vg^T \vx 
&= \max_{\vx \in \mathcal{C}} \vg^T \sum_{G \in \G} \vw_G \\
\notag
&= \max_{\vx \in \mathcal{C}} \sum_{G \in \G} \vg^T \vw_G ~\ \textbf{s.t. } \vx = \sum_{G \in \G}  \vw_G  \\
\label{returnhere}
&= \max_{\{ \vw_G \} \in \mathcal{W}(\vx)} \sum_{G \in \G} \vg^T \vw_G ~\ \textbf{s.t. } \sum_{G \in \G} \| \vw_G \|_2 + \fll \| \vw_G \|_1 \leq \sqrt{k} (1 + \lambda_1 )  \\
\notag
&\stackrel{(i)}{\leq}  \max_{\{ \vw_G \} \in \W(\vx)} \sum_{G \in \G} \vg^T \vw_G ~\ \textbf{s.t. } \sum_{G \in \G}(1 + \fll) \| \vw_G \|_2 \leq \sqrt{k} (1 + \lambda_1 )  \\
\label{usecorr}
& =  \max_{\{ \vw_G \} \in \W(\vx)} \sum_{G \in \G} \vg^T \vw_G  ~\ \textbf{s.t. } \sum_{G \in \G} \| \vw_G \|_2 \leq  \frac{\sqrt{kl} (1 + \lambda_1 )}{ (\sqrt{l} + \lambda_1)} \\
\notag
&\stackrel{(ii)}{=} \frac{\sqrt{kl} (1 + \lambda_1 )}{\sqrt{l} + \lambda_1} ~\ \max_{G \in \G} \| \vg_G \|_2  \\
\label{useforcorr}
&\leq \sqrt{k} \left( 1 + \lambda_1 \right) ~\ \max_{G \in \G} \| \vg_G \|_2
\end{align}
where we define $\vg_G$ to be the sub vector of $\vg$ indexed by group $G$. (i) follows since the constraint set is a superset of the constraint in the expression above it, from the fact that $ \| \va \|_2 \leq \| \va\|_1 ~\ \forall \va$,   and (ii) is a result of simple convex analysis. 

The mean width is then bounded as
\begin{equation}
\label{mwnosquare}
\omega(\mathcal{C}) \leq \sqrt{k} \left( 1 + \lambda_1 \right)  ~\ \mathbb{E} \left[ \max_{G \in \G} \| \vg_G \|_2 \right] . 
\end{equation}
Squaring both sides of (\ref{mwnosquare}), we get
\begin{align}
\notag
\omega(\mathcal{C})^2 &\leq k \left( 1 + \lambda_1 \right)^2 \left[ \mathbb{E} [ \max_{G \in \G} \| \vg_G \|_2 ] \right]^2 \\
\notag
& \stackrel{(iii)}{\leq} k \left( 1 + \lambda_1 \right)^2 ~\ \mathbb{E} \left[ \left(\max_{G \in \G} \| \vg_G \|_2 \right)^2 \right] \\
\notag
&\stackrel{(iv)}{=} k \left( 1 + \lambda_1 \right)^2 ~\ \mathbb{E} \left[ \max_{G \in \G} \| \vg_G \|_2^2 \right]
\end{align}
where $(iii)$ follows from Jensen's inequality and (iv) follows from the fact that the square of the maximum of non negative numbers is the same as the maximum of the squares. Now, note that since $\vg$ is Gaussian, $\| \vg_G \|^2$ is a $\chi^2$ random variable with at most $B$ degrees of freedom.  From Lemma \ref{lem:chisq}, we have 
\begin{equation}
\label{mwset}
\omega(\mathcal{C})^2 \leq k \left( 1 + \lambda_1 \right)^2 (\sqrt{2 \log(K)} + \sqrt{L})^2.
\end{equation}

This gives us one of the two terms in the $\min\{ \cdot, \cdot \}$ in the statement of the Lemma. Since $\alpha$ is bounded away from $0$, we can treat the term in the parenthesis as a constant. For the second term, let us revisit (\ref{returnhere}), and obtain the following inequalities:
\begin{align}
\notag
\max_{\vx \in \mathcal{C}} \vg^T \vx &= \max_{\{ \vw_G \} \in \mathcal{W}(\vx)} \sum_{G \in \G} \vg^T \vw_G ~\ \textbf{s.t. } \sum_{G \in \G} \| \vw_G \|_2 + \fll \| \vw_G \|_1 \leq \sqrt{k} (1 + \lambda_1 )  \\
\notag
&\stackrel{(v)}{\leq} \max_{\{ \vw_G \} \in \mathcal{W}(\vx)} \sum_{G \in \G} \vg^T \vw_G ~\ \textbf{s.t. } \sum_{G \in \G} \frac{1}{\sqrt{L}} \| \vw_G \|_1 + \fll \| \vw_G \|_1 \leq \sqrt{k} (1 + \lambda_1) \\
\notag 
&= \max_{\{ \vw_G \} \in \mathcal{W}(\vx)} \sum_{G \in \G} \vg^T \vw_G ~\ \textbf{s.t. } \left( \frac{\sqrt{\alpha} + \lambda_1}{\sqrt{l}} \right) \sum_{G \in \G} \| \vw_G \|_1 \leq \sqrt{k} (1 + \lambda_1) \\
\notag
&= \max_{\{ \vw_G \} \in \mathcal{W}(\vx)} \sum_{G \in \G} \vg^T \vw_G ~\ \textbf{s.t. } \sum_{G \in \G} \| \vw_G \|_1\leq  \frac{\sqrt{kl} (1 + \lambda_1)}{\sqrt{\alpha} + \lambda_1} \\
\notag
&=  \frac{\sqrt{kl} (1 + \lambda_1)}{\sqrt{\alpha} + \lambda_1} \max_{G \in \G} \max_{i \in G} | (\vg_G)_i | \\
\notag
&\leq  \frac{\sqrt{kl} (1 + \lambda_1)}{ \lambda_1} \max_{i} | \vg_i |
\end{align}

Where the constraint set in $(v)$ is a superset of that in the statement above it. Again, after squaring both sides, taking expectations and applying Jensen's inequality, 

\[
\omega(\mathcal{C})^2 \leq kl \left( \frac{1 + \lambda_1}{\lambda_1} \right)^2 \mathbb{E}\left[ \max_i \vg_i^2 \right]
\]
The quantity inside the expectation is a $\chi^2$ variable with one degree of freedom, and from Lemma \ref{lem:chisq}, we obtain
\[
\omega(\mathcal{C})^2 \leq C kl  \log{(p)} . 
\]

This gives the second term in the $\min\{ \cdot, \cdot \}$, and finishes the proof

\end{proof}
Lemma \ref{lem:mw} and Theorem~\ref{thmplan} lead directly to Theorem \ref{thmmain}. 

The results in the proof above shed some more light on our regularizer $h(\vx)$. If $\lambda_1 = 0$, then the problem reduces to that of classification using the overlapping group lasso penalty, and we obtain the corresponding sample complexity bound. For simple sparsity without any structure, we would want $\lambda_1$ to be large, in which case $\frac{1 + \lambda_1}{ \lambda_1} \rightarrow 1$, and $(1 + \lambda_1) \rightarrow \infty$. This would then entail the bounds for the $\ell_1$ regularized problem taking over, keeping all other parameters $k, K, l ,L$ fixed. 

%%%%%%%%%%%%%%%%%%%%%%%%%%%%%%%%%%%%%%%%%%%%%%%%%%
%%%%%%%%%%%%%%%%%%%%%%%%%%%%%%%%%%%%%%%%%%%%%%%%%%
%%%%%%%%%%%%%%%%%%%%%%%%%%%%%%%%%%%%%%%%%%%%%%%%%%
%%%%%%%%%%%%%%%%%%%%%%%%%%%%%%%%%%%%%%%%%%%%%%%%%%
%%%%%%%%%%%%%%%%%%%%%%%%%%%%%%%%%%%%%%%%%%%%%%%%%%
%%%%%%%%%%%%%%%%%%%%%%%%%%%%%%%%%%%%%%%%%%%%%%%%%%
%%%%%%%%%%%%%%%%%%%%%%%%%%%%%%%%%%%%%%%%%%%%%%%%%%
%%%%%%%%%%%%%%%%%%%%%%%%%%%%%%%%%%%%%%%%%%%%%%%%%%
%%%%%%%%%%%%%%%%%%%%%%%%%%%%%%%%%%%%%%%%%%%%%%%%%%
%%%%%%%%%%%%%%%%%%%%%%%%%%%%%%%%%%%%%%%%%%%%%%%%%%
%%%%%%%%%%%%%%%%%%%%%%%%%%%%%%%%%%%%%%%%%%%%%%%%%%

%%%%%%%%%%%%%%%%%%%%%%%%%%%%%%%%%%%%%%%%%%%%%%%%%%
%%%%%%%%%%%%%%%%%%%%%%%%%%%%%%%%%%%%%%%%%%%%%%%%%%
%%%%%%%%%%%%%%%%%%%%%%%%%%%%%%%%%%%%%%%%%%%%%%%%%%
%%%%%%%%%%%%%%%%%%%%%%%%%%%%%%%%%%%%%%%%%%%%%%%%%%
%%%%%%%%%%%%%%%%%%%%%%%%%%%%%%%%%%%%%%%%%%%%%%%%%%
%%%%%%%%%%%%%%%%%%%%%%%%%%%%%%%%%%%%%%%%%%%%%%%%%%
%%%%%%%%%%%%%%%%%%%%%%%%%%%%%%%%%%%%%%%%%%%%%%%%%%
%%%%%%%%%%%%%%%%%%%%%%%%%%%%%%%%%%%%%%%%%%%%%%%%%%
%%%%%%%%%%%%%%%%%%%%%%%%%%%%%%%%%%%%%%%%%%%%%%%%%%
%%%%%%%%%%%%%%%%%%%%%%%%%%%%%%%%%%%%%%%%%%%%%%%%%%
%%%%%%%%%%%%%%%%%%%%%%%%%%%%%%%%%%%%%%%%%%%%%%%%%%

\subsection{Extensions to Data with Correlated Entries}
\label{sec:corr}
The results we proved above can be extended to data $\mPhi$ with correlated Gaussian entries as well (see \citep{reccorrelated} for results in linear regression settings). Indeed, in most practical applications we are interested in, the features are expected to contain correlations. For example, in the fMRI application that is one of the major motivating applications of our work, it is reasonable to assume that voxels in the brain will exhibit correlation amongst themselves at a given time instant. This entails scaling the number of measurements by the condition number of the covariance matrix $\mtx\Sigma$, where we assume that each row if the measurement matrix $\mPhi$ is sampled from a Gaussian $(0,\mtx\Sigma)$ distribution. Specifically, we obtain a generalization of the result in \citep{plan1bit} for the SOGlasso with a correlated Gaussian design. 

%%%%%%%%%%%%%%%%%%%%%%%%%%%%%%%%%%%%%%%%%%%%%%%%%%
%%%%%%%%%%%%%%%%%%%%%%%%%%%%%%%%%%%%%%%%%%%%%%%%%%
%%%%%%%%%%%%%%%%%%%%%%%%%%%%%%%%%%%%%%%%%%%%%%%%%%
%%%%%%%%%%%%%%%%%%%%%%%%%%%%%%%%%%%%%%%%%%%%%%%%%%
%%%%%%%%%%%%%%%%%%%%%%%%%%%%%%%%%%%%%%%%%%%%%%%%%%
%%%%%%%%%%%%%%%%%%%%%%%%%%%%%%%%%%%%%%%%%%%%%%%%%%
%%%%%%%%%%%%%%%%%%%%%%%%%%%%%%%%%%%%%%%%%%%%%%%%%%
%%%%%%%%%%%%%%%%%%%%%%%%%%%%%%%%%%%%%%%%%%%%%%%%%%
%%%%%%%%%%%%%%%%%%%%%%%%%%%%%%%%%%%%%%%%%%%%%%%%%%
%%%%%%%%%%%%%%%%%%%%%%%%%%%%%%%%%%%%%%%%%%%%%%%%%%
%%%%%%%%%%%%%%%%%%%%%%%%%%%%%%%%%%%%%%%%%%%%%%%%%%
%%%%%%%%%%%%%%%%%%%%%%%%%%%%%%%%%%%%%%%%%%%%%%%%%%
%%%%%%%%%%%%%%%%%%%%%%%%%%%%%%%%%%%%%%%%%%%%%%%%%%
%%%%%%%%%%%%%%%%%%%%%%%%%%%%%%%%%%%%%%%%%%%%%%%%%%
%%%%%%%%%%%%%%%%%%%%%%%%%%%%%%%%%%%%%%%%%%%%%%%%%%
%%%%%%%%%%%%%%%%%%%%%%%%%%%%%%%%%%%%%%%%%%%%%%%%%%
%%%%%%%%%%%%%%%%%%%%%%%%%%%%%%%%%%%%%%%%%%%%%%%%%%
%%%%%%%%%%%%%%%%%%%%%%%%%%%%%%%%%%%%%%%%%%%%%%%%%%
%%%%%%%%%%%%%%%%%%%%%%%%%%%%%%%%%%%%%%%%%%%%%%%%%%
%%%%%%%%%%%%%%%%%%%%%%%%%%%%%%%%%%%%%%%%%%%%%%%%%%
%%%%%%%%%%%%%%%%%%%%%%%%%%%%%%%%%%%%%%%%%%%%%%%%%%
%%%%%%%%%%%%%%%%%%%%%%%%%%%%%%%%%%%%%%%%%%%%%%%%%%

We now consider the following constraint set:
\begin{equation}
\label{corr}
\mathcal{C}_{corr} = \{ \vx: h(\vx) \leq \frac{1}{\sigma_{min}(\mtx\Sigma^{\frac12})}\sqrt{k} (1 + \lambda_1), \|\mtx\Sigma^{\frac12} \vx \| \leq 1 \} . 
\end{equation}

We consider the set $\mathcal{C}_{corr}$ and not $\mathcal{C}$ in (\ref{constraint_set}), since we  require the constraint set to be a subset of the unit Euclidean ball. In the proof of Corollary \ref{correlated} below, we will reduce the problem to an optimization over variables of the form $\vz = \mtx\Sigma^\frac12 \vx$, and hence we require $\| \mtx{\Sigma}^{\frac{1}{2}} \vx \|_2 \leq 1$. Enforcing this constraint leads to the corresponding upper bound on $h(\vx)$. 

We now obtain the following generalization of Theorem \ref{thmmain}, for correlated data

\begin{corollary}
\label{correlated}
Let the entries of the data matrix $\mPhi$ be sampled from a $\N(\vct{0}, \mtx{\Sigma})$ distribution. Suppose the measurements follow the model in (\ref{ymodel}). Suppose we wish to recover a $(k, l)-$ group sparse vector from the set $\mathcal{C}_{corr}$ in (\ref{corr}). Suppose the true coefficient vector $\vx^\star$ satisfies $\| \mtx\Sigma^{\frac12} \vx^\star \|  = 1$. Then, so long as the number of measurements $n$ satisfies
\[
n \geq C \frac{\sigma^2}{\epsilon^2} \kappa(\mtx\Sigma) k \min\{ \log(K) + L, l \log(p) \} , 
\]
the solution to (\ref{optgen}) satisfies
\[
\| \hat{\vx} - \vx^\star \|^2 \leq \frac{\epsilon}{\sigma_{min}(\mtx\Sigma) } . 
\]

where $\sigma_{min}(\cdot)$, $\sigma_{max}(\cdot)$ and $\kappa(\cdot)$ denote the minimum and maximum singular values and the condition number of the corresponding matrices respectively.

\end{corollary}

We prove this result in Appendix \ref{app:correlated}. The proof is a straightforward modification of the proof of Theorem \ref{thmmain}.  A similar result along the lines of Theorem \ref{thmnewbound} can also be proved.

%%%%%%%%%%%%%%%%%%%%%%%%%%%%%%%%%%%%%%%%%%%%%%%%%%
%%%%%%%%%%%%%%%%%%%%%%%%%%%%%%%%%%%%%%%%%%%%%%%%%%
%%%%%%%%%%%%%%%%%%%%%%%%%%%%%%%%%%%%%%%%%%%%%%%%%%
%%%%%%%%%%%%%%%%%%%%%%%%%%%%%%%%%%%%%%%%%%%%%%%%%%
%%%%%%%%%%%%%%%%%%%%%%%%%%%%%%%%%%%%%%%%%%%%%%%%%%
%%%%%%%%%%%%%%%%%%%%%%%%%%%%%%%%%%%%%%%%%%%%%%%%%%
%%%%%%%%%%%%%%%%%%%%%%%%%%%%%%%%%%%%%%%%%%%%%%%%%%
%%%%%%%%%%%%%%%%%%%%%%%%%%%%%%%%%%%%%%%%%%%%%%%%%%
%%%%%%%%%%%%%%%%%%%%%%%%%%%%%%%%%%%%%%%%%%%%%%%%%%
%%%%%%%%%%%%%%%%%%%%%%%%%%%%%%%%%%%%%%%%%%%%%%%%%%
%%%%%%%%%%%%%%%%%%%%%%%%%%%%%%%%%%%%%%%%%%%%%%%%%%

%%%%%%%%%%%%%%%%%%%%%%%%%%%%%%%%%%%%%%%%%%%%%%%%%%
%%%%%%%%%%%%%%%%%%%%%%%%%%%%%%%%%%%%%%%%%%%%%%%%%%
%%%%%%%%%%%%%%%%%%%%%%%%%%%%%%%%%%%%%%%%%%%%%%%%%%
%%%%%%%%%%%%%%%%%%%%%%%%%%%%%%%%%%%%%%%%%%%%%%%%%%
%%%%%%%%%%%%%%%%%%%%%%%%%%%%%%%%%%%%%%%%%%%%%%%%%%
%%%%%%%%%%%%%%%%%%%%%%%%%%%%%%%%%%%%%%%%%%%%%%%%%%
%%%%%%%%%%%%%%%%%%%%%%%%%%%%%%%%%%%%%%%%%%%%%%%%%%
%%%%%%%%%%%%%%%%%%%%%%%%%%%%%%%%%%%%%%%%%%%%%%%%%%
%%%%%%%%%%%%%%%%%%%%%%%%%%%%%%%%%%%%%%%%%%%%%%%%%%
%%%%%%%%%%%%%%%%%%%%%%%%%%%%%%%%%%%%%%%%%%%%%%%%%%
%%%%%%%%%%%%%%%%%%%%%%%%%%%%%%%%%%%%%%%%%%%%%%%%%%

\section{Applications and Experiments}
\label{sec:expts}
In this section, we perform experiments on both real and toy data, and show that the function proposed in (\ref{eq:reg}) indeed recovers the kind of sparsity patterns we are interested in in this paper. First, we experiment with some toy data to understand the properties of the function $h(\vx)$ and in turn, the solutions that are yielded from the optimization problem (\ref{optgen}). Here, we take the opportunity to report results on linear regression problems as well. We then present results using two datasets from cognitive neuroscience and computational biology.

\subsection{The SOGlasso for Multitask Learning}
\label{sec:mtl}
The SOG lasso is motivated in part by multitask learning applications. The group lasso is a commonly used tool in multitask learning, and it encourages the same set of features to be selected across all tasks. As mentioned before, we wish to focus on a less restrictive version of multitask learning, where the main idea is to encourage sparsity patterns that are similar, but not identical, across tasks. Such a restriction corresponds to a scenario where the different tasks are related to each other, in that they use similar features, but are not exactly identical. This is accomplished by defining subsets of similar features and searching for solutions that select only a few subsets (common across tasks) and a sparse number of features within each subset (possibly different across tasks). Figure \ref{fig:spatterns} shows an example of the patterns that typically arise in sparse multitask learning applications, along with the one we are interested in. We see that the SOGlasso, with it's ability to select a few groups and only a few non zero coefficients within those groups lends itself well to the scenario we are interested in.

In the multitask learning setting, suppose the features are give by $\mPhi_t$, for tasks $t = \{ 1,2, \ldots, \T \}$, and corresponding sparse vectors $\vx^\star_t \in \R^p$. These vectors can be arranged as columns of a matrix $\mX^\star$. Suppose we are now given $M$ groups $\tilde{\G} = \{ \tilde{G}_1, \tilde{G}_2, \ldots \}$ with maximum size $\tilde{B}$. Note that the groups will now correspond to sets of rows of $\mX^\star$. 

Let $\vx^\star = [\vx_1^{\star T} ~\ \vx_2^{\star T} ~\ \ldots \vx_\T^{\star T}]^T \in \R^{\T p}$, and $\vy  = [\vy_1^T ~\ \vy_2^T ~\ \ldots \vy_\T^T]^T \in \R^{\T n}$. We also define $ \G = \{G_1, G_2 , \ldots, G_M \}$ to be the set of groups defined on $ \R^{\T p} $ formed by aggregating the rows of $\mX$ that were originally in $\tilde{\G}$, so that $\vx$ is composed of groups $G \in \G$, and let the corresponding maximum group size be $B = \T \tilde{B}$. By organizing the coefficients in this fashion, we can reduce the multitask learning problem into the standard form as considered in (\ref{optgen}). Hence, all the results we obtain in this paper can be extended to the multitask learning setting as well. 

\subsubsection{Results on fMRI dataset}
In this experiment, we compared SOGlasso, lasso, standard  multitask Glasso (with each feature grouped across tasks), the overlapping group lasso \citep{jacob} (with the same groups as in SOGlasso) and the Elastic Net \citep{enet} in analysis of the star-plus dataset \citep{mitchellfmri}. 6 subjects made judgements that involved processing 40 sentences and 40 pictures while their brains were scanned in half second intervals using fMRI\footnote{Data and documentation available at http://www.cs.cmu.edu/afs/cs.cmu.edu/project/theo-81/www/}. We retained the 16 time points following each stimulus, yielding 1280 measurements at each voxel. The task is to distinguish, at each point in time, which kind of stimulus a subject was processing.  \citep{mitchellfmri} showed that there exists cross-subject consistency in the cortical regions useful for prediction in this task. Specifically, experts partitioned each dataset into 24 non overlapping regions of interest (ROIs), then reduced the data by discarding all but 7 ROIs and, for each subject, averaging the BOLD response across voxels within each ROI. With the resulting data, the authors showed that a classifier trained on data from 5 participants generalized above chance when applied to data from a 6th--thus proving some degree of consistency across subjects in how the different kinds of information were encoded.

We assessed whether SOGlasso could leverage this cross-individual consistency to aid in the discovery of predictive voxels without requiring expert pre-selection of ROIs, or data reduction, or any alignment of voxels beyond that existing in the raw data. Note that, unlike \citep{mitchellfmri}, we do not aim to learn a solution that generalizes to a withheld subject. Rather, we aim to discover a group sparsity pattern that suggests a similar set of voxels in all subjects, before optimizing a separate solution for each individual. If SOGlasso can exploit cross-individual anatomical similarity from this raw, coarsely-aligned data, it should show reduced cross-validation error relative to the lasso applied separately to each individual. If the solution is sparse within groups and highly variable across individuals, SOGlasso should show reduced cross-validation error relative to Glasso. Finally, if SOGlasso is finding useful cross-individual structure, the features it selects should align at least somewhat with the expert-identified ROIs shown by \citep{mitchellfmri} to carry consistent information.

We trained the 5 classifiers using 4-fold cross validation to select the regularization parameters, considering all available voxels without preselection.  We group regions of $5\times5\times1$ voxels and considered overlapping groups ``shifted" by 2 voxels in the first 2 dimensions.\footnote{The irregular group size compensates for voxels being larger and scanner coverage being smaller in the z-dimension (only 8 slices relative to 64 in the x- and y-dimensions).}  

Figure \ref{errperperson5} shows the prediction error (misclassification rate) of each classifier for every individual subject. SOGlasso shows the smallest error. The substantial gains over lasso indicate that the algorithm is successfully leveraging cross-subject consistency in the location of the informative features, allowing the model to avoid over-fitting individual subject data. We also note that the SOGlasso classifier, despite being trained without any voxel pre-selection, averaging, or alginment, performed comparably to the best-performing classifier reported by  \cite{mitchellfmri}, which was trained on features average over 7 expert pre-selected ROIs 

To assess how well the clusters selected by SOGlasso align with the anatomical regions thought a-priori to be involved in sentence and picture representation, we calculated the proportion of selected voxels falling within the 7 ROIs identified by \citep{mitchellfmri} as relevant to the classification task (Table \ref{tabroi}). For SOGlasso an average of 61.9\% of identified voxels fell within these ROIs, significantly more than for lasso, group lasso (with or without overlap) and the elastic net. The overlapping group lasso, despite returning a very large number or predictors, hardly overlaps with the regions of interest to cognitive neuroscientists. The lasso and the elastic net make use of the fact that a separate classifier can be trained for each subject, but even in this case, the overlap with the regions of interest is low. The group lasso also fares badly in this regard, since the same voxels are forced to be selected across individuals, and this means that the regions of interest which will be misaligned across subjects will not in general be selected for each subject. All these drawbacks are circumvented by the SOGlasso. This shows that even without expert knowledge about the relevant regions of interest, our method partially succeeds in isolating the voxels that play a part in the classification task.
\begin{table}
\centering
\begin{tabular}{|| c |  c | c |c ||}
\hline
\textbf{Method} & \textbf{Avg. Overlap with ROI $\%$} \\
\hline
OGlasso &  27.18 \\
\hline
ENet &  43.46 \\
\hline
Lasso &  41.51\\
\hline
Glasso &  47.43 \\
\hline
SOGlasso &  61.90 \\
\hline
\end{tabular}  
\caption{Mean Sparsity levels of the methods considered, and the average overlap with the precomputed ROIs in \citep{mitchellfmri}}
\label{tabroi}
\end{table}

\begin{figure}[!t]
\centering
\includegraphics[width = 80mm, height = 60mm]{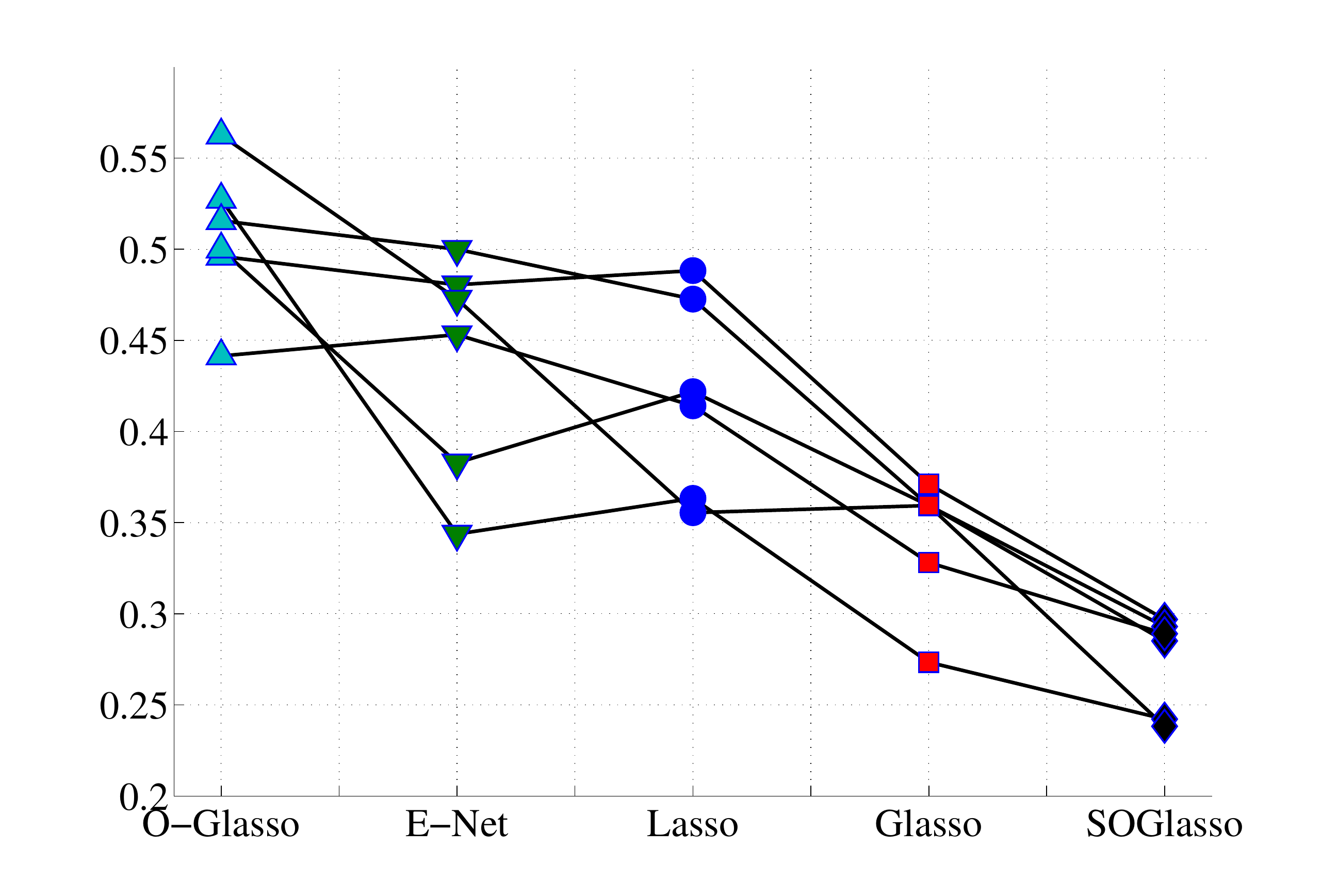}
\caption{Misclassification error on a hold out set for different methods, on a per subject basis. Each solid line connects the different errors obtained for a particular subject in the dataset. }
\label{errperperson5}
 \end{figure}

We make the following observations from Figure \ref{errperperson5} and Figure \ref{sparsity_patterns}
\begin{itemize}
\item The overlapping group lasso \citep{jacob} is ill suited for this problem. This is natural, since the premise is that the brains of different subjects can only be crudely aligned, and the overlapping group lasso will force the same voxel to be selected across all individuals. It will also force all the voxels in a group to be selected, which is again undesirable from our perspective. This leads to a high number of voxels selected, and a high error. 
\item The elastic net \citep{enet} treats each subject independently, and hence does not leverage the inter-subject similarity that we know exists across brains. The fact that all correlated voxels are also picked, coupled with a highly noisy signal means that a large number of voxels are selected, and this not only makes the result hard to interpret, but also leads to a large generalization error. 
\item The lasso \citep{tibshirani} is similar to the elastic net in that it does not leverage the inter subject similarities. At the same time, it enforces sparsity in the solutions, and hence a fewer number of voxels are selected across individuals. It allows any task correlated voxel to be selected, regardless of its spatial location, and that leads to a highly distributed sparsity pattern (Figure \ref{lassomontage}). It leads to a higher cross-validation error, indicating that the ungrouped voxels are inferior predictors. Like the elastic net, this leads to a  poor generalization error (Figure \ref{errperperson5}). The distributed sparsity pattern, low overlap with predetermined Regions of Interest, and the high error on the hold out set is what we believe makes the lasso a suboptimal procedure to use.
\item The group lasso \citep{lounicimtl} groups a single voxel across individuals. This allows for taking into account the similarities between subjects, but not the minor differences across subjects. Like the overlapping group lasso, if a voxel is selected for one person, the same voxel is forced to be selected for all people. This means, if a voxel encodes picture or sentence in a particular subject, then the same voxel is forced to be selected across subjects, and can arbitrarily encode picture or sentence. This gives rise to a purple haze in Figure \ref{glassomontage}, and makes the result hard to interpret. The purple haze manifests itself due to the large number of ambiguous voxels in Figure \ref{percentages}.
\item Finally, the SOGlasso as we have argued helps in accounting for both the similarities and the differences across subjects. This leads to the learning of a code that is at the same time very sparse and hence interpretable, and leads to an error on the test set that is the best among the different methods considered. The SOGlasso (Figure \ref{soslassomontage}) overcomes the drawbacks of lasso and Glasso by allowing different voxels to be selected per group. This gives rise to a spatially clustered sparsity pattern, while at the same time selecting a negligible amount of voxels that encode both picture and sentences (Figure \ref{percentages}). Also, the resulting sparsity pattern has a larger overlap with the ROI's than other methods considered. 
\end{itemize}

\begin{figure}
\subfigure[Lasso]{
\includegraphics[width = 70mm, height = 50mm]{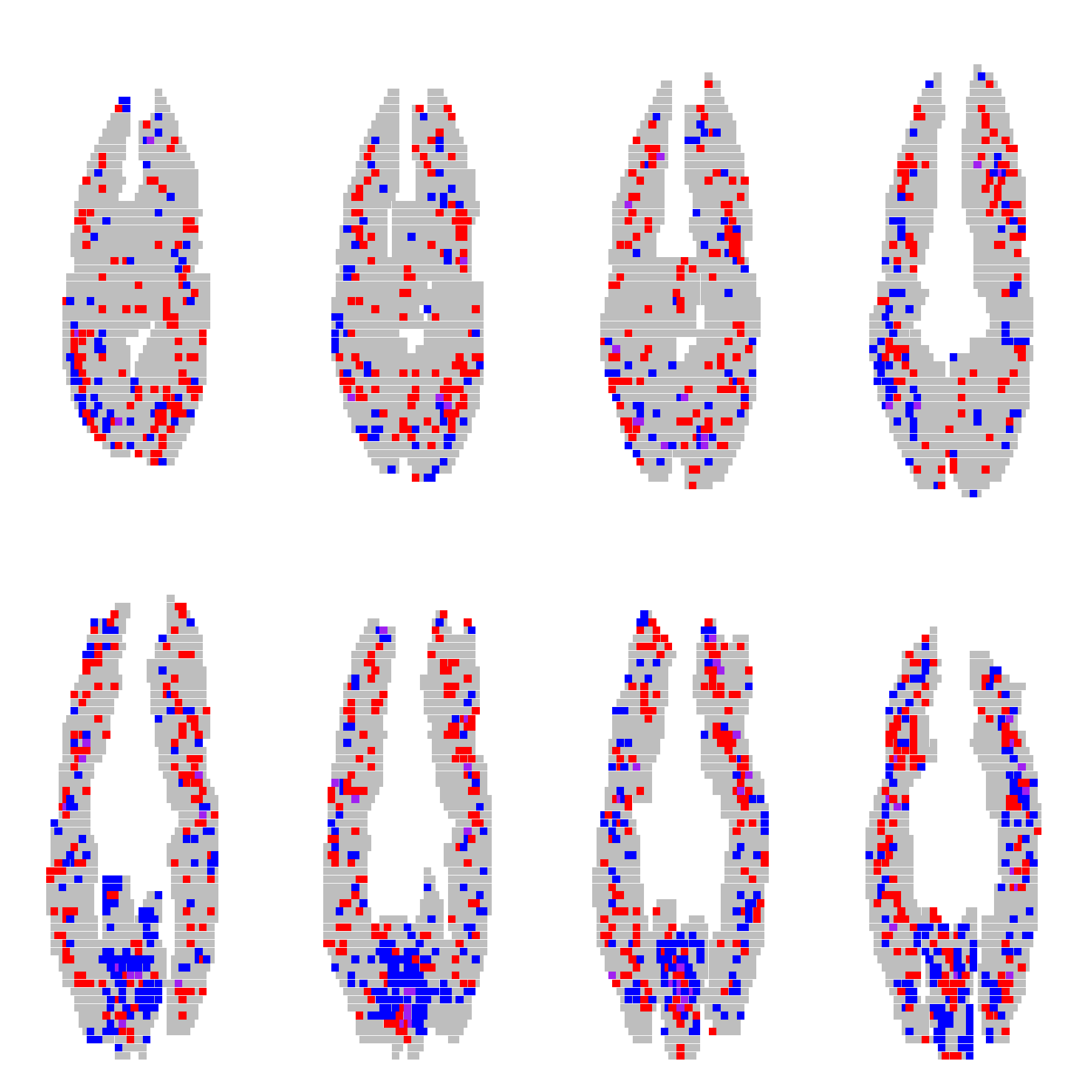}
\label{lassomontage}
}
\subfigure[Group Lasso]{
\includegraphics[width = 70mm, height = 50mm]{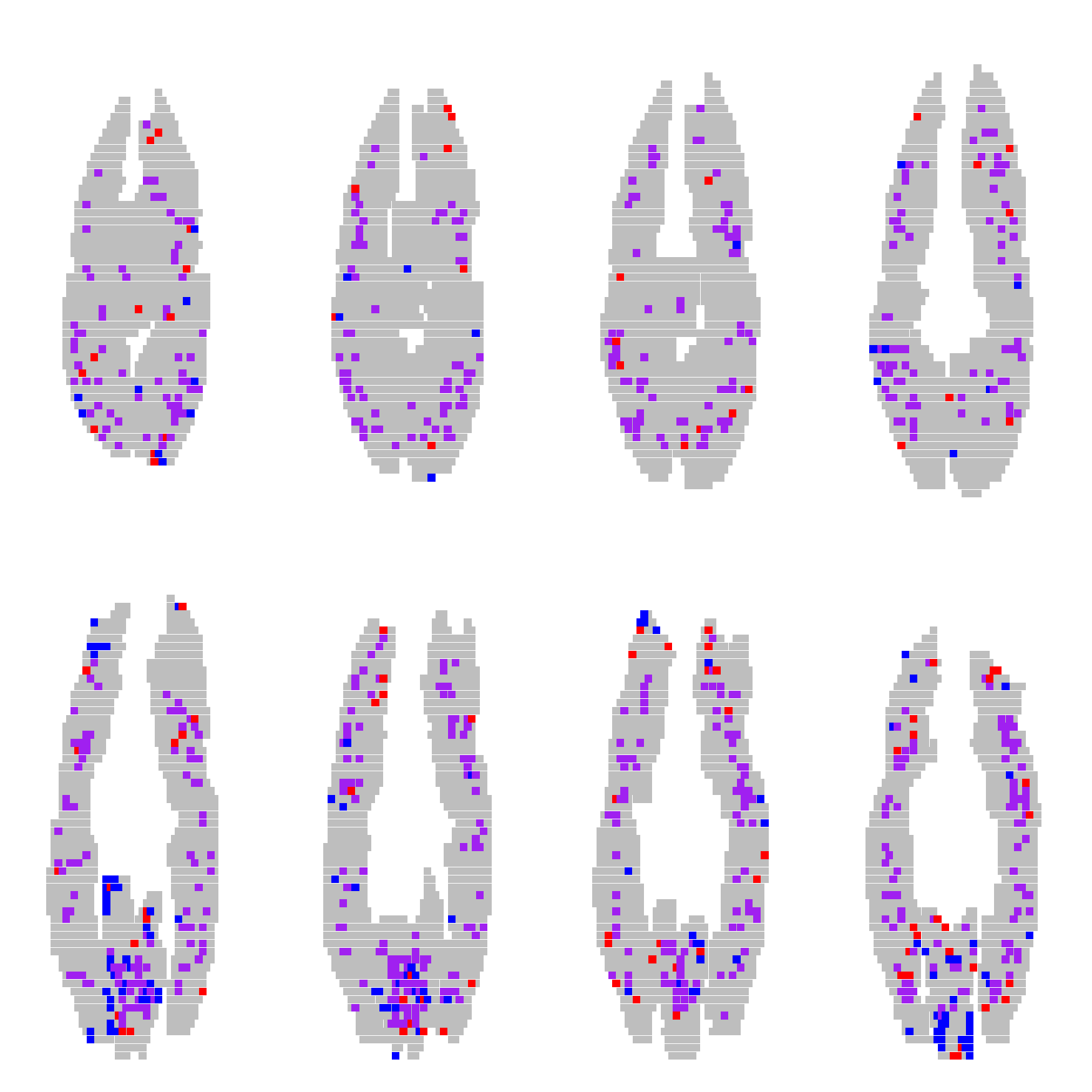}
\label{glassomontage}
}
\subfigure[SOG Lasso]{
\includegraphics[width = 70mm, height = 50mm]{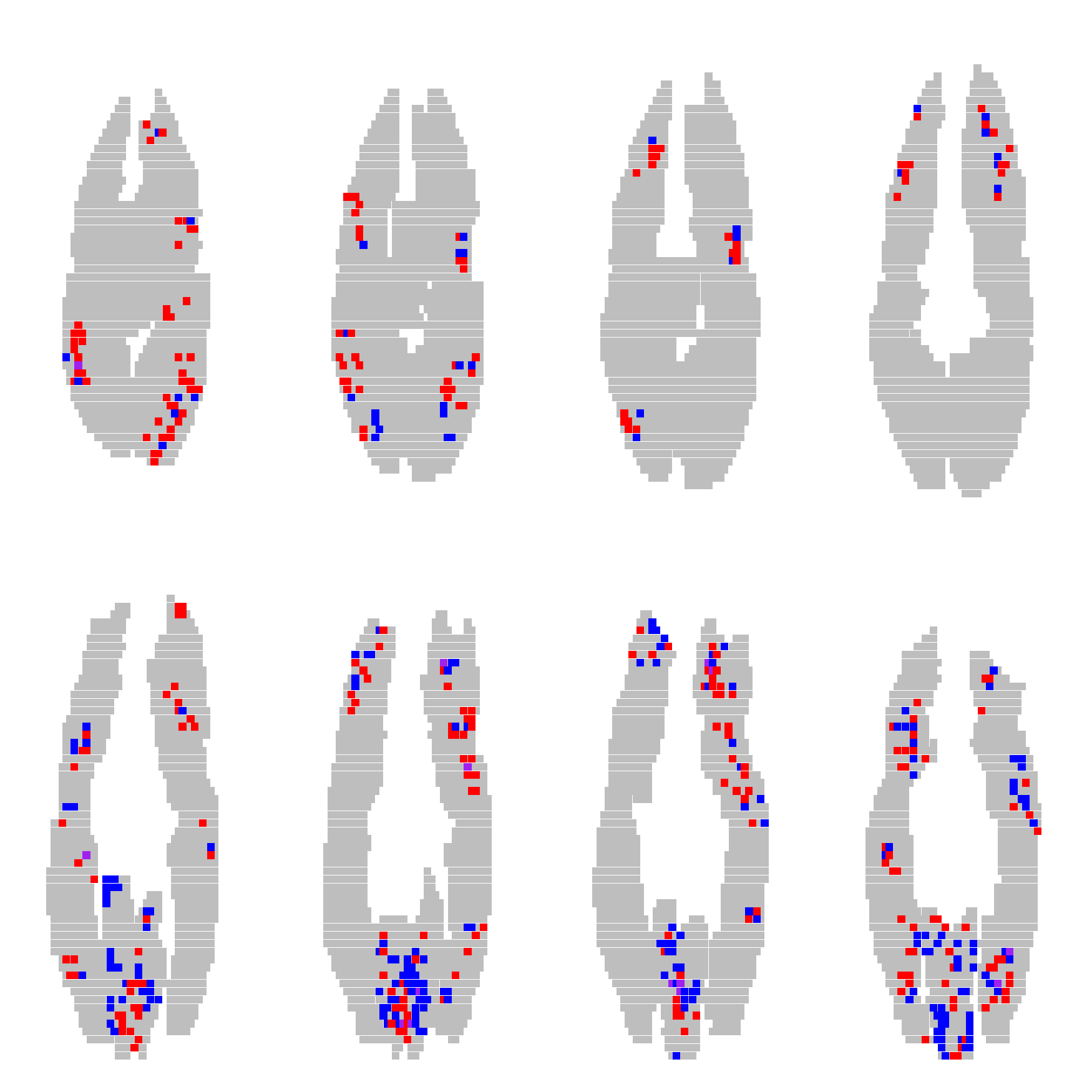}
\label{soslassomontage}
}
\subfigure[Voxel Encodings ($\%$)]{
\includegraphics[width = 80mm, height = 45mm]{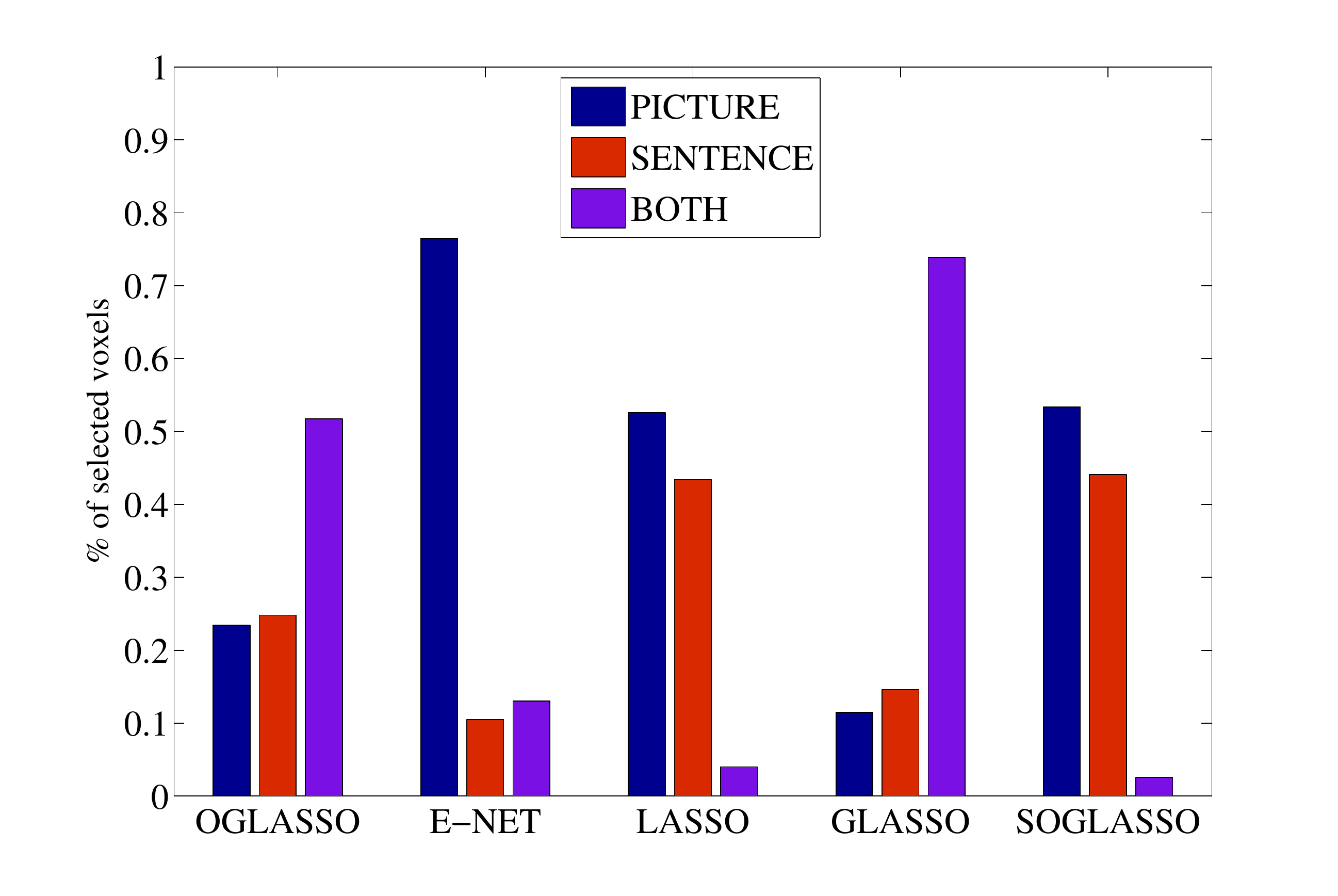}
\label{percentages}
}
\caption{[Best seen in color]. Aggregated sparsity patterns across subjects per brain slice. All the voxels selected across subjects in each slice are colored in red, blue or purple. Red indicates voxels that exhibit a picture response in at least one subject and never exhibit a sentence response. Blue indicates the opposite. Purple indicates voxel that exhibited a a picture response in at least one subject and a sentence response in at least one more subject. \subref{percentages} shows the percentage of selected voxels that encode picture, sentence or both. }
\label{sparsity_patterns}
\end{figure}

\subsection{Toy Data, Linear Regression}
Although not the primary focus of this paper, we show that the method we propose can also be applied to the linear regression setting. To this end, we consider simulated data and a multitask linear regression setting, and look to recover the coefficient matrix. We also use the simulated data to study the properties of the function we propose in (\ref{eq:reg}). 

The toy data is generated as follows: we consider $\T = 20$ tasks, and consider overlapping groups of size $B = 6$. The groups are defined so that neighboring groups overlap ($G_1 = \{1,2, \ldots, 6 \}$, $G_2 = \{ 5,6,\ldots,10 \}$, $G_3 = \{9,10,\ldots,14  \}$, $\dots$). We consider a case with $M = 100$ groups, 
We set $k = 10$ groups to be active. We vary the sparsity level of the active groups $\alpha$ and obtain $m = 100$ Gaussian linear measurements corrupted with Additive White Gaussian Noise of standard deviation $\sigma = 0.1$. We repeat this procedure 100 times and average the results. To generate the coefficient matrices $\mX^\star$, we select $k$ groups at random, and within the active groups, only retain  fraction $\alpha$ of the coefficients, again at random. The retained locations are then populated with uniform $[-1 , 1]$ random variables. 

The regularization parameters were clairvoyantly picked to minimize the Mean Squared Error (MSE) over a range of parameter values. The results of applying lasso, standard latent group lasso \citep{jacob}, Group lasso where each group corresponds to a row of the sparse matrix, \citep{lounicimtl} and our SOGlasso to these data are plotted in Figures \ref{varyalphagauss}, varying $\alpha$. 
\begin{figure}[!h]
\centering
\subfigure[Varying $\alpha$]{
\includegraphics[width = 65mm, height = 50mm]{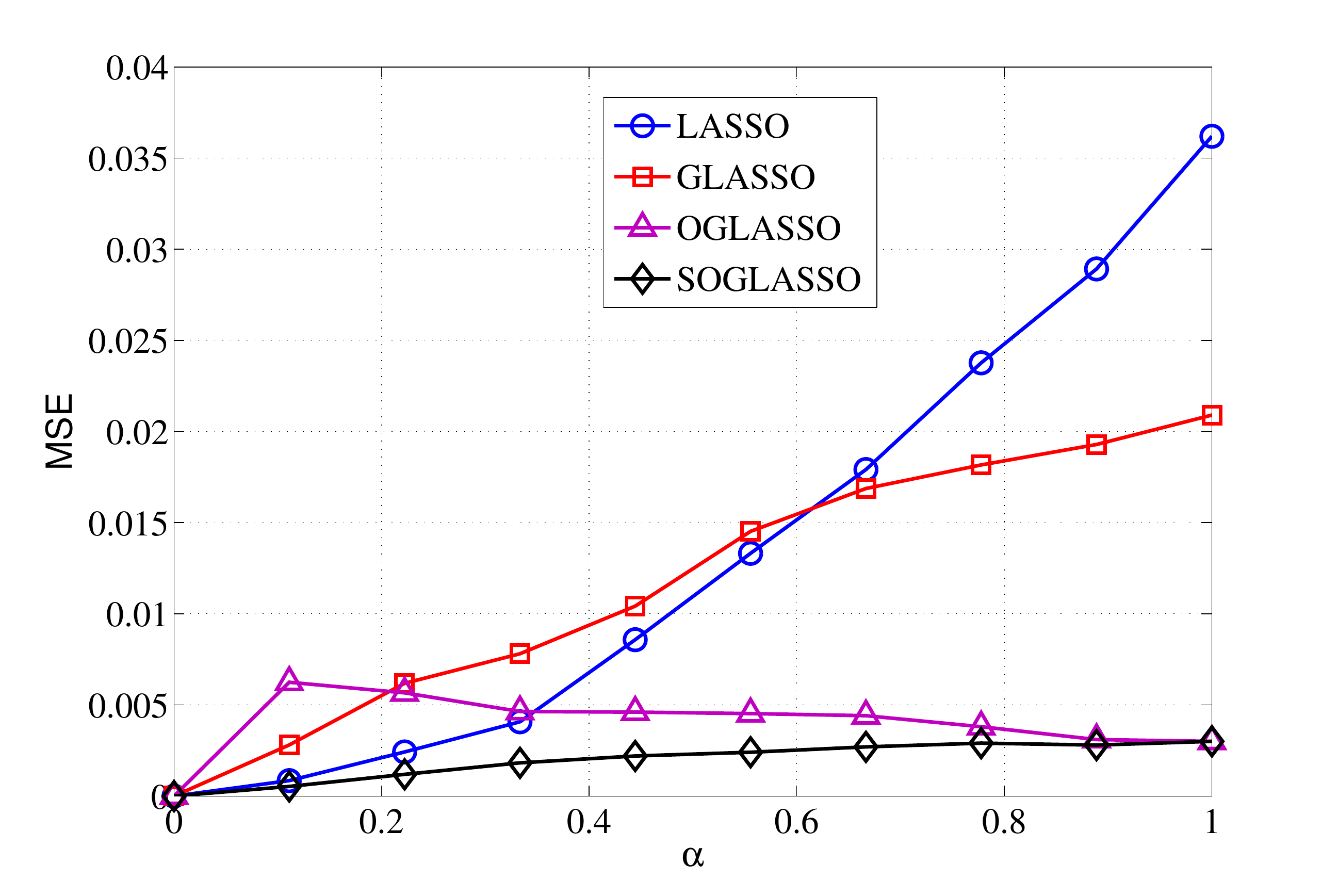}
\label{varyalphagauss}}
\subfigure[Sample pattern]{
\includegraphics[width =27mm, height = 44mm]{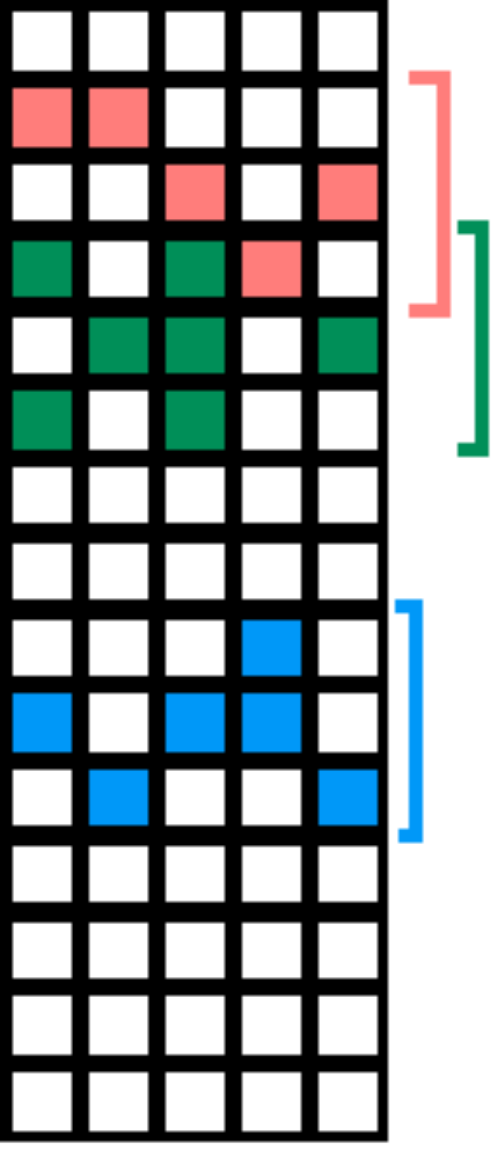}
\label{toy}}
\caption{Figure \subref{varyalphagauss} shows the result of varying $\alpha$. The SOGlasso accounts for both inter and intra group sparsity, and hence performs the best. The Glasso achieves good performance only when the active groups are non sparse. Figure \subref{toy} shows a toy sparsity pattern, with different colors and brackets denoting different overlapping groups}
\label{fig:toy}
\end{figure}

Figure \ref{varyalphagauss} shows that, as the sparsity within the active group reduces (i.e. the active groups become more dense), the overlapping group lasso performs progressively better. This is because the overlapping group lasso does not account for sparsity within groups, and hence the resulting solutions are far from the true solutions for small values of $\alpha$. The SOGlasso however does take this into account, and hence has a lower error when the active groups are sparse. Note that as $\alpha \rightarrow 1$, the SOGlasso approaches O-Glasso \citep{jacob}. The Lasso \citep{tibshirani} does not account for group structure at all and performs poorly when $\alpha$ is large, whereas the Group lasso \citep{lounicimtl} does not account for overlapping groups, and hence performs worse than O-Glasso and SOGlasso.

\subsection{SOGlasso for Gene Selection}
As explained in the introduction, another motivating application for the SOGlasso arises in computational biology, where one needs to predict whether a particular breast cancer tumor will lead to metastasis or not, from gene expression profiles. We used the breast cancer dataset compiled by \citep{breast_data} and grouped the genes into pathways as in \citep{pathway}. To make the dataset balanced, we perform a 3-way replication of one of the classes as in \citep{jacob}, and also restrict our analysis to genes that are atleast in one pathway. Again as in \citep{jacob}, we ensure that all the replicates are in the same fold for cross validation. We do not perform any preprocessing of the data, other than the replication to balance the dataset. We compared our method to the standard lasso, and the overlapping group lasso. The standard group lasso \citep{yuanlin} is ill-suited for this experiment, since the groups overlap and the sparsity pattern we expect is a union of groups, and it has been shown that the group lasso method will not recover the signal in such cases. 

\begin{table}[!h]
\centering
 \begin{tabular}{ || c |c || }
 \hline
 \textbf{Method} & \textbf{Misclassification Rate } \\
  \hline                       
  lasso & 0.42   \\
  \hline
  OGlasso \citep{jacob} & 0.39 \\
  \hline  
  SOGlasso &  0.33 \\
  \hline
\end{tabular}
 \captionof{table}{Misclassification Rate on the test set for the different methods considered. The SOGlasso obtained better error rates as compared to the other methods.}
   \label{bcancer}
\end{table} 

We trained a model using 4-fold cross validation on 80$\%$ of the data, and used the remaining 20$\%$ as a final test set. Table \ref{bcancer} shows the results obtained. We see that the SOGlasso penalty leads to lower classification errors as compared to the lasso or the latent group lasso. The errors reported are the ones obtained on the final (held out) test set.  We refrain from performing simple ridge regression \citep{ridge} since the data is high dimensional, and by not enforcing sparsity in the solution, the result will be un-interpretable. 

%%%%%%%%%%%%%%%%%%%%%%%%%%%%%%%%%%%%%%%%%%%%%%
%%%%%%%%%%%%%%%%%%%%%%%%%%%%%%%%%%%%%%%%%%%%%%
%%%%%%%%%%%%%%%%%%%%%%%%%%%%%%%%%%%%%%%%%%%%%%
%%%%%%%%%%%%%%%%%%%%%%%%%%%%%%%%%%%%%%%%%%%%%%
%%%%%%%%%%%%%%%%%%%%%%%%%%%%%%%%%%%%%%%%%%%%%%
%%%%%%%%%%%%%%%%%%%%%%%%%%%%%%%%%%%%%%%%%%%%%%
%%%%%%%%%%%%%%%%%%%%%%%%%%%%%%%%%%%%%%%%%%%%%%
%%%%%%%%%%%%%%%%%%%%%%%%%%%%%%%%%%%%%%%%%%%%%%
%%%%%%%%%%%%%%%%%%%%%%%%%%%%%%%%%%%%%%%%%%%%%%
%%%%%%%%%%%%%%%%%%%%%%%%%%%%%%%%%%%%%%%%%%%%%%

\section{Conclusions}
\label{sec:conc}
In this paper, we introduced a function that can be used to constrain solutions of high dimensional feature selection problems so that they display both within and across group sparsity. We generalized the sparse group lasso to cases with arbitrary overlap, and proved consistency results in a classification setting. Our results unify the results between the lasso and the group lasso (with or without overlap), and reduce to those cases as special cases. We also outlined the use of the function in multitask fMRI and computational biology problems. Moreover, we make minimal assumptions on the model the generates data, and hence our results can be seen in a very general light.

From an algorithmic standpoint, when the groups overlap a lot, the replication procedure used in this paper might not be memory efficient. Future work involves designing algorithms that preclude replication, while at the same time allowing for the SOG- sparsity patterns to be generated.

From a cognitive neuroscience point of view, future work involves grouping the voxels in more intelligent ways. Our method to group spatially co-located voxels yields results that are significantly better than traditional lasso-based methods, but it remains to be seen whether there are better motivated ways to group them. For example, one might consider grouping voxels based on functional connectivities, or take into account the geodesic distance on the brain surface. 

%%%%%%%%%%%%%%%%%%%%%%%%%%%%%%%%%%%%%%%%%%%%%%
%%%%%%%%%%%%%%%%%%%%%%%%%%%%%%%%%%%%%%%%%%%%%%
%%%%%%%%%%%%%%%%%%%%%%%%%%%%%%%%%%%%%%%%%%%%%%
%%%%%%%%%%%%%%%%%%%%%%%%%%%%%%%%%%%%%%%%%%%%%%
%%%%%%%%%%%%%%%%%%%%%%%%%%%%%%%%%%%%%%%%%%%%%%
%%%%%%%%%%%%%%%%%%%%%%%%%%%%%%%%%%%%%%%%%%%%%%
%%%%%%%%%%%%%%%%%%%%%%%%%%%%%%%%%%%%%%%%%%%%%%
%%%%%%%%%%%%%%%%%%%%%%%%%%%%%%%%%%%%%%%%%%%%%%
%%%%%%%%%%%%%%%%%%%%%%%%%%%%%%%%%%%%%%%%%%%%%%
%%%%%%%%%%%%%%%%%%%%%%%%%%%%%%%%%%%%%%%%%%%%%%

\appendix

\section{}
\label{app:mwnonconv}
We bound the mean width of the set in (\ref{nonconvset}),

\begin{proof}[Proof of Lemma \ref{mwnonconvset}]
Since $\| \vx \|_2 \leq 1$, 
\begin{equation}
\label{max1}
\max \vg^T\vx = \max_{S \in \mathcal{S}} \| \vg_S \|_2, 
\end{equation}
where $\mathcal{S}$ is the set of indices given by
\[
\mathcal{S} = \{ S_i \subset \{1,2, \cdots, p \} : ~\ j \in S_i ~\textbf{if} ~\vx \in \mathcal{{C}}_{nc}, ~\ \vx_j \neq 0  \} . 
\]
The cardinality of $\mathcal{S}$ is bounded as
\begin{align}
\notag
| \mathcal{S} | &= \binom{K}{k} \binom{kL}{k l} \\
\notag
&\leq \left( \frac{eK}{k} \right)^k \left( \frac{eL}{l} \right)^{kl} \\
\label{sbound}
\Rightarrow \log(|\mathcal{S}|)&\leq k \left[ \log{\left(\frac{K}{k}\right)} + l \log{\left(\frac{L}{l}\right)} +2 \right] . 
\end{align}
From (\ref{max1}) we have 
\begin{align}
\notag
\mathbb{E} [\max \vg^T \vx]^2 &= \mathbb{E} \max_{S \in \mathcal{S}} \| \vg_S \|^2 \\
\notag
&\leq \left( \sqrt{\log{| \mathcal{S} |}} + \sqrt{|S|} \right)^2  \\
\label{mwnonconv}
&\leq C k \left[ \log{\left( \frac{K}{k} \right)} + l \log{\left( \frac{L}{l} \right)} + l + 2 \right] , 
\end{align}
where the first inequality follows from Lemma 3.2 in \citep{nraistats} and the last inequality follows from (\ref{sbound}). 
\end{proof}

%%%%%%%%%%%%%%%%%%%%%%%%%%%%%%%%%%%%%%%%%%%%%%
%%%%%%%%%%%%%%%%%%%%%%%%%%%%%%%%%%%%%%%%%%%%%%
%%%%%%%%%%%%%%%%%%%%%%%%%%%%%%%%%%%%%%%%%%%%%%
%%%%%%%%%%%%%%%%%%%%%%%%%%%%%%%%%%%%%%%%%%%%%%
%%%%%%%%%%%%%%%%%%%%%%%%%%%%%%%%%%%%%%%%%%%%%%
%%%%%%%%%%%%%%%%%%%%%%%%%%%%%%%%%%%%%%%%%%%%%%
%%%%%%%%%%%%%%%%%%%%%%%%%%%%%%%%%%%%%%%%%%%%%%
%%%%%%%%%%%%%%%%%%%%%%%%%%%%%%%%%%%%%%%%%%%%%%
%%%%%%%%%%%%%%%%%%%%%%%%%%%%%%%%%%%%%%%%%%%%%%
%%%%%%%%%%%%%%%%%%%%%%%%%%%%%%%%%%%%%%%%%%%%%%

\section{}
\label{app:correlated}
We prove the result in Corollary \ref{correlated}. Before we do so, we state and prove a Lemma

\begin{lemma}
\label{singvalbound}
Suppose $\mtx{A} \in \R^{s \times t}$, and let $\mtx{A}_G \in \R^{|G| \times t}$ be the sub matrix of $\mtx{A}$ formed by retaining the rows indexed by group $G \in \G$. Suppose $\sigma_{max}(\mtx{A})$ is the maximum singular value of $\mtx{A}$, and similarly for $\mtx{A}_G$. Then
\[
\sigma_{max}(\mtx{A}) \geq \sigma_{max}(\mtx{A}_G) ~\ \forall G \in \G . 
\]
\end{lemma}
\begin{proof}
Consider an arbitrary vector $\vx \in \R^p$, and let $\bar{G}$ be the indices that are to indexed by $G$. We then have the following:
\begin{align}
\notag
\| \mtx{A}\vx \|^2 &= \left\| \left[ \begin{array} {c} \mtx{A}_G \vx \\ \mtx{A}_{\bar{G}} \vx \end{array} \right] \right\|^2 \\
\notag
&= \| \mtx{A}_G \vx \|^2 + \| \mtx{A}_{\bar{G}} \vx \|^2 \\
\label{lemint}
\Rightarrow \| \mtx{A} \vx\|^2 &\geq \| \mtx{A}_G \vx \|^2
\end{align}
We therefore have
\begin{align}
\notag
\sigma_{max}(\mtx{A}) &= \sup_{\| \vx \| = 1} \| \mtx{A} \vx \| \\
\notag
&\geq \sup_{\| \vx \| = 1} \| \mtx{A}_G \vx \| \\
\notag
&= \sigma_{max}(A_G)
\end{align}
where the inequality follows from (\ref{lemint}).
\end{proof}

We now proceed to prove Corollary \ref{correlated}. 
\begin{proof}
Since the entries of the data matrices are correlated Gaussians, the inner products in the objective function of the optimization problem (\ref{optgen}) can be written as 
\[
\langle \mPhi_i , \vx \rangle =  \langle \mtx\Sigma^{\frac12} \mPhi'_i , \vx \rangle = \langle  \mPhi'_i , \mtx\Sigma^{\frac12} \vx \rangle, 
\]
where $\mPhi'_i \sim \N(\vct{0}, \mtx{I})$. Hence, we can replace $\vx$ in our result in Theorem \ref{thmmain} by $\mtx\Sigma^{\frac12} \vx$, and make appropriate changes to the constraint set. 

We then see that the optimization problem we need to solve is
\[
\hat{\vx} = \arg \min_{\vx} - \sum_{i=1}^n \vy_i \langle \mPhi'_i, \mtx\Sigma^{\frac12} \vx \rangle ~\textbf{ s.t. } ~\vx \in \mathcal{C}_{corr} . 
\]
Defining $\vz  = \mtx\Sigma^{\frac12} \vx$, we can equivalently write the above optimization as
\begin{equation}
\label{repz}
\hat{\vz} = \arg \min - \sum_{i =1}^n \vy_i \langle \mPhi'_i, \vz \rangle ~\textbf{ s.t. } ~\vz \in \mtx\Sigma^{\frac12} \mathcal{C}_{corr}, 
\end{equation}
where we define $\Sigma^{\frac12} \mathcal{C}$ to be the set $\mathcal{C}v$,  with each element multiplied by $\mtx\Sigma^{\frac12}$.  We see that (\ref{repz}) is of the same form as (\ref{optgen}), with the constraint set ``scaled" by the matrix $\mtx\Sigma^{\frac12}$. We now need to bound the mean width of the set $\mtx\Sigma^{\frac12}\mathcal{C}_{corr}$.

We then have
\begin{align*}
\max_{\vz \in \mtx\Sigma^{\frac12}\mathcal{C}_{corr}} \vg^T\vz &= \max_{\vx \in \mathcal{C}} \vg^T\mtx\Sigma^{\frac12}\vx  \\
&= \max_{\vx \in \mathcal{C}_{corr}} (\mtx\Sigma^{\frac12} \vg)^T\vx \\
&\leq \frac{\sqrt{k} (1 + \lambda_1 )}{\sigma_{min}(\mtx\Sigma^{\frac12}) } \max_{G \in \G} \| \mtx\Sigma_G^{\frac12}\vg \|
\end{align*}

where the final inequality follows from the exact same arguments used to obtain (\ref{useforcorr}). By $\mtx\Sigma_G^{\frac12}$, we mean the $|G| \times p$ sub matrix of $\mtx\Sigma^{\frac12}$ obtained by retaining rows indexed by group $G$. 

To compute the mean width, we need to find $\E[\max_{G \in G} \| \mtx\Sigma_G^{\frac12}\vg  \|^2 ]$. Now, since $\vg \sim \N(0, \mtx{I})$, $\mtx\Sigma_G^{\frac12}\vg \sim \N(0, \mtx\Sigma_G^{\frac12} (\mtx\Sigma_G^{\frac12})^T)$. Hence, $ \| \mtx\Sigma_G^{\frac12}\vg \|^2 \leq  \sigma_{max}(\mtx\Sigma_G^{\frac12}) \| \vct{c} \|^2 $ where $\vct{c} \sim \N(0,\mtx{I}_{|G|})$. $\| \vct{c} \|^2 \sim \chi^2_{|G|}$, and we can again use Lemma \ref{lem:chisq} to obtain the following bound for the mean width:

\begin{align}
\notag
\omega(\mtx\Sigma^{\frac12} \mathcal{C})^2 &\leq \frac{k (1 + \lambda_1 )^2}{\sigma_{min}(\mtx\Sigma^{\frac12})^2 } (\sqrt{2 \log(K)} + \sqrt{L})^2  \left[ \max_{G \in \G} \sigma_{\max}(\mtx\Sigma_G) \right] \\
\notag
&\leq \sigma_{max}(\mtx\Sigma) \frac{k (1 + \lambda_1 )^2}{\sigma_{min}(\mtx\Sigma) } (\sqrt{2 \log(K)} + \sqrt{L})^2 \\
\label{corrmw1}
&\leq C \kappa(\mtx\Sigma) k (\log(K) + L)
\end{align}

Similarly, following a procedure similar to that used to prove Theorem \ref{thmmain}, we obtain
\beq
\label{corrmw}
\omega(\mtx\Sigma^{\frac12} \mathcal{C})^2 \leq C \kappa(\mtx\Sigma)k \min\{\log(K) + L, l \log(p) \}, 
\eeq

where the last inequality follows from Lemma \ref{singvalbound}.

We then have that so long as the number of measurements $n$ is larger than $C \frac{\sigma^2}{\epsilon^2}$ times the quantity in (\ref{corrmw}), 
\[
\left\| \hat{\vz} - \vz^\star \right\|^2 = \left\| \mtx\Sigma^{\frac12}\hat{\vx} - \mtx\Sigma^{\frac12}\vx^\star \right\|^2 \leq \epsilon . 
\]

However, note that
\begin{equation}
\label{sigbound}
\sigma_{min}(\mtx\Sigma) \| \hat{\vx} - \vx^\star \|^2 \leq \left\| \mtx\Sigma^{\frac12}\hat{\vx} - \mtx\Sigma^{\frac12}\vx^\star \right\|^2 . 
\end{equation}

(\ref{corrmw}) and (\ref{sigbound}) combine to give the final result. Note that for the sake of keeping the exposition simple, we have used Lemma \ref{singvalbound} and bounded the number of sufffieicnt measurements  as a function of the maximum singular value of $\mtx\Sigma$. However, the number of measurements  only depends on $\max_{G \in \G} \sigma_{max}(\mtx\Sigma_G)$, which is typically much lesser. 

\end{proof}

\bibliographystyle{plainnat}
\bibliography{SOSLASSO_NR3}

\end{document}